\documentclass{article}

\usepackage{microtype}
\usepackage{graphicx}
\usepackage{subfigure}
\usepackage{booktabs} 

\usepackage{hyperref}
\usepackage{enumitem}
\usepackage{longtable}
\usepackage{multirow}
\usepackage{makecell}
\usepackage[table]{xcolor}

\usepackage[ruled,vlined]{algorithm2e}

\usepackage[accepted]{icml2024}

\usepackage{amsmath}
\usepackage{amssymb}
\usepackage{mathtools}
\usepackage{amsthm}

\usepackage[capitalize,noabbrev]{cleveref}

\theoremstyle{plain}
\newtheorem{theorem}{Theorem}[section]

\theoremstyle{definition}

\newtheorem{assumption}[theorem]{Assumption}
\theoremstyle{remark}

\icmltitlerunning{Personalized Federated Learning for Spatio-Temporal Forecasting: A Dual Semantic Alignment-Based Contrastive Approach}
\begin{document}

\twocolumn[
\icmltitle{Personalized Federated Learning for Spatio-Temporal Forecasting: A Dual Semantic Alignment-Based Contrastive Approach}




\begin{icmlauthorlist}
\icmlauthor{Qingxiang Liu}{yyy}
\icmlauthor{Sheng Sun}{yyy}
\icmlauthor{Yuxuan Liang}{zzz}
\icmlauthor{Jingjing Xue}{yyy}
\icmlauthor{Min Liu}{yyy}
\end{icmlauthorlist}

\icmlaffiliation{yyy}{Institute of Computing Technology, Chinese Academy of Sciences, Beijing, China}
\icmlaffiliation{zzz}{The Hong Kong University of Science and Technology (Guangzhou)}

\icmlcorrespondingauthor{Min Liu}{liumin@ict.ac.cn}

\icmlkeywords{Contrastive Learning,
Federated Learning,
Spatio-temporal Forecasting,
Time Series Analysis}

\vskip 0.3in
]



\printAffiliationsAndNotice  

\begin{abstract}
The existing federated learning (FL) methods for spatio-temporal forecasting fail to capture the inherent spatio-temporal heterogeneity, which calls for personalized FL (PFL) methods to model the spatio-temporally variant patterns. While contrastive learning approach is promising in addressing spatio-temporal heterogeneity, the existing methods are noneffective in determining negative pairs and can hardly apply to PFL paradigm. To tackle this limitation, we propose a novel PFL method, named \underline{\textbf{F}}ederated d\underline{\textbf{U}}al s\underline{\textbf{E}}mantic a\underline{\textbf{L}}ignment-based contra\underline{\textbf{S}}tive learning (FUELS), which can adaptively align positive and negative pairs based on semantic similarity, thereby injecting precise spatio-temporal heterogeneity into the latent representation space by auxiliary contrastive tasks. From temporal perspective, a hard negative filtering module is introduced to dynamically align heterogeneous temporal representations for the supplemented intra-client contrastive task. From spatial perspective, we design lightweight-but-efficient prototypes as client-level semantic representations, based on which the server evaluates spatial similarity and yields client-customized global prototypes for the supplemented inter-client contrastive task. Extensive experiments demonstrate that FUELS outperforms state-of-the-art methods, with communication cost decreasing by around 94\%.

\end{abstract}

\section{Introduction}
Spatio-temporal forecasting aims at predicting the future trends with historical records from distributed facilities.
Actually, the facilities may be deployed by different organizations and therefore the collected data can hardly be uploaded to a public server for training due to copyright protection.
For example, in wireless traffic prediction, different telecom operators may hesitate to share access to traffic data of base stations (BSs) to others.
Federated Learning (FL) enables clients to collaboratively optimize a shared model without the disclosure of private data and simultaneously can achieve comparable performance with centralized learning methods \cite{mcmahan2017communication}.
Therefore, FL generates great promise in the task of spatio-temporal forecasting and many sophisticated FL methods have been proposed to effectively evaluate the spatio-temporal correlation \cite{li2021privacy,li2022federated}. 
Given that the server has only access to local model parameters, these FL methods cannot exploit a joint spatio-temporal correlation-capturing module like centralized methods \cite{li2018dcrnn_traffic,li2021spatial}.
In general, most of them decouple the spatio-temporal correlation, and evaluate spatial correlation in the server by Graph Neural Networks \cite{zhao2020cellular,lin2021data} and temporal correlation at local clients by Recurrent Neural Networks \cite{zhu2021joint,mahdy2020clustering}.
\begin{figure}[!t]
    \vskip 0.1in
    \begin{center}
        \centering
        \includegraphics[width=0.9\linewidth]{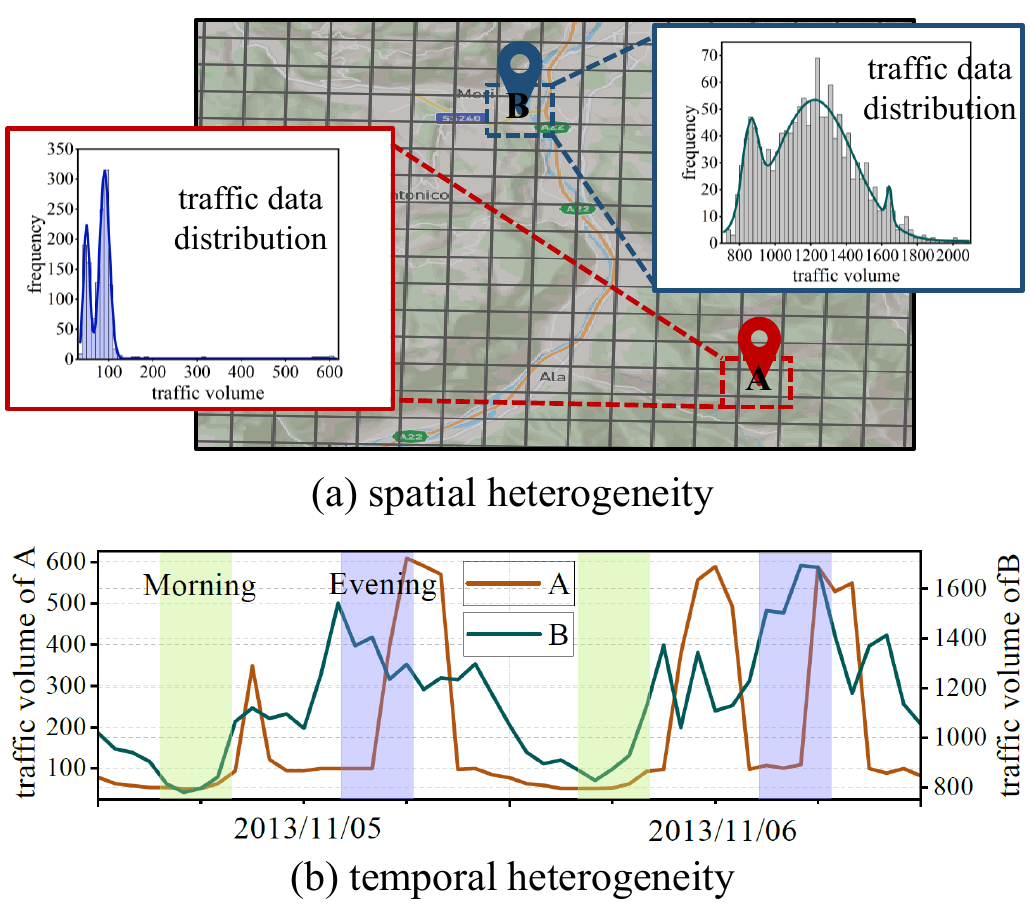}
        \caption{Spatio-temporal heterogeneity inside traffic flows.}
        \label{motivation}
    \end{center}
    \vskip -0.1in
\end{figure}

Nevertheless, these FL forecasting methods cannot capture the inherent spatio-temporal heterogeneity by adopting shared model parameters across clients.
As shown in Figure \ref{motivation}, traffic patterns are spatio-variant (\textit{traffic distribution of two clients varies}) and time-variant (\textit{traffic patterns in the morning and evening are diverse}).
The spatio-temporal heterogeneous traffic patterns call for personalized federated learning (PFL) methods to tackle with differentiated representations in the latent space from spatio-temporal level.
However, current PFL methods all focus on increasing the generalization ability of local models by sharing diverse knowledge across clients \cite{karimireddy2020scaffold,tan2022towards,chen2023efficient}, which can hardly work on coupled spatio-temporal heterogeneity.


Contrastive learning has emerged as an effective technique which shows great promise in improving PFL paradigm to address spatio-temporal heterogeneity.
Generally, contrastive learning can maximize the agreement between similar semantic representations (\textit{positive pairs}), which can increase representations' differentiation, thereby injecting spatio-temporal heterogeneity into the latent representation space.
Since there are no explicit semantic labels in spatio-temporal forecasting like the task of image classification, it is a challenge to determine positive and negative pairs.
In \cite{ji2023spatio}, for spatial heterogeneity, learnable embeddings are introduced for clustering regions' representations, which can hardly apply to PFL paradigm, due to the exposure of raw data and frequent communication in the process of parameter optimization.
For temporal heterogeneity, region-level and city-level representations of the same time stamps make up positive pairs, otherwise negative pairs, which ignores semantic similarity among representations of different time stamps (termed \textit{hard negatives}).
Therefore, \textbf{it remains a question how to exactly determine negative and positive pairs when adopting contrastive learning in PFL paradigm for tackling spatio-temporal heterogeneity.}

To this end, we propose a novel PFL method, named \underline{\textbf{F}}ederated d\underline{\textbf{U}}al s\underline{\textbf{E}}mantic a\underline{\textbf{L}}ignment-based contra\underline{\textbf{S}}tive learning (FUELS), which can dynamically establish positive and negative pairs based on spatio-temporal representations' similarity, hence guaranteeing the validity of spatio-temporal heterogeneity modeling.
From temporal perspective, we propose a hard negative filtering module for each client, which can adaptively align true negative pairs for intra-client contrastive task so as to inject temporal heterogeneity into temporal representations.
From spatial perspective, we define prototypes as client-level semantic representations and directly serve as the communication carrier, which enables sharing denoising knowledge across clients in a communication-efficient manner. We propose a Jensen Shannon Divergence (JSD)-based aggregation mechanism, which firstly aligns homogeneous and heterogeneous prototypes from clients and then yields client-customized global positive and negative prototypes for inter-client contrastive task to increase the spatial differentiation of local models.

We summarize the key contributions as follows.
\begin{itemize}[leftmargin=*]
    \item 
    To the best of our knowledge, this is the first PFL method towards spatio-temporal heterogeneity, where local training is enhanced by two well-crafted contrastive loss items to increase prediction models' ability of discerning spatial and temporal heterogeneity.
    \item We design a hard negative filtering module to dynamically evaluate the heterogeneous degree among temporal representations and establish true negative pairs. 
    \item We propose a JSD-based aggregation mechanism to generate particular global positive and negative prototypes for clients, striking a balance between sharing semantic knowledge and evaluating spatial heterogeneity.
    \item We validate the effectiveness and efficiency of the proposed FUELS from both theoretical analysis and extensive experiments. The numerical results show that FUELS can achieve consistently superior forecasting performance and simultaneously decrease the communication cost by around 94\%.
\end{itemize}

\section{Problem Formulation}

One typical task of spatio-temporal forecasting is wireless traffic prediction. 
Given $N$ BSs as clients in the FL paradigm, the $n$-th BS has its traffic observations $\mathcal{V}_n = \left\{ v_n^k \left| k \in [1, K] \right. \right\}$, where $v_n^k \in \mathbb{R}$ denotes the detected traffic volume of client $n$ at the $k$-th time stamp.
Based on sliding window mechanism \cite{zhang2021dual}, $\mathcal{V}_n$ can be divided into input-output pairs (samples), which is denoted as $\mathcal{D}_n = \left\{ \left( \textbf{x}_n^k; y_n^k   \right) \right\} $.
Thereinto, $y_n^k = v_n^k$ denotes the value to be predicted.
$\textbf{x}_n^k = \left( \textbf{cv}_n^k, \textbf{pv}_n^k \right)$, where $\textbf{cv}_n^k \in \mathbb{R}^{c}, \textbf{pv}_n^k \in \mathbb{R}^{q}$, and $\textbf{x}_n^k \in \mathbb{R}^{c+q}$.
$\textbf{cv}_n^k = \left(v_n^{k-c}, \cdots, v_n^{k-2}, v_n^{k-1} \right)$ and $\textbf{pv}_n^k =(v_n^{k-qp}, \cdots$, $ v_n^{k-2p}, v_n^{k-p})$ denote two historical traffic sequences to evaluate the closeness and periodicity of traffic data.
$c$, $q$, and $p$ represent the size of close window, the size of periodic window, and the periodicity of traffic volume.

Since BSs are deployed at different locations, such as downtown area and suburb \emph{etc.}, the detected traffic data across different BSs have diverse distributions. 
Therefore, in PFL paradigm, each BS tends to learn its own prediction model for performance improvement.
The objective can be formulated as
\begin{equation}
    \mathop {\min }\limits_{\left\{ {w_n} \right\}} \frac{1}{N}\sum\limits_{n = 1}^N {\frac{{|{{\mathcal{D}}_n}|}}{D}} \sum\limits_{\left( {{\textbf{x}}_n^k,y_n^k} \right) \in {{\mathcal{D}}_n}} {{\cal L}\left( {f\left( {{w_n};{\textbf{x}}_n^k} \right),y_n^k} \right)}.
\end{equation}
where $f(\cdot)$ represents prediction model. $\mathcal{L}(\cdot)$ denotes the loss function.
$D = \sum\nolimits_{n = 1}^N {\left| {{\mathcal{D}_n}} \right|} $ denotes the total number of samples over all $N$ clients. $w_n (1\le n\le N)$ denotes the model parameters of client $n$.

\section{Methodology}
\begin{figure*}[!ht]
\vskip 0.1in
    \centering
    \includegraphics[width=0.95\linewidth]{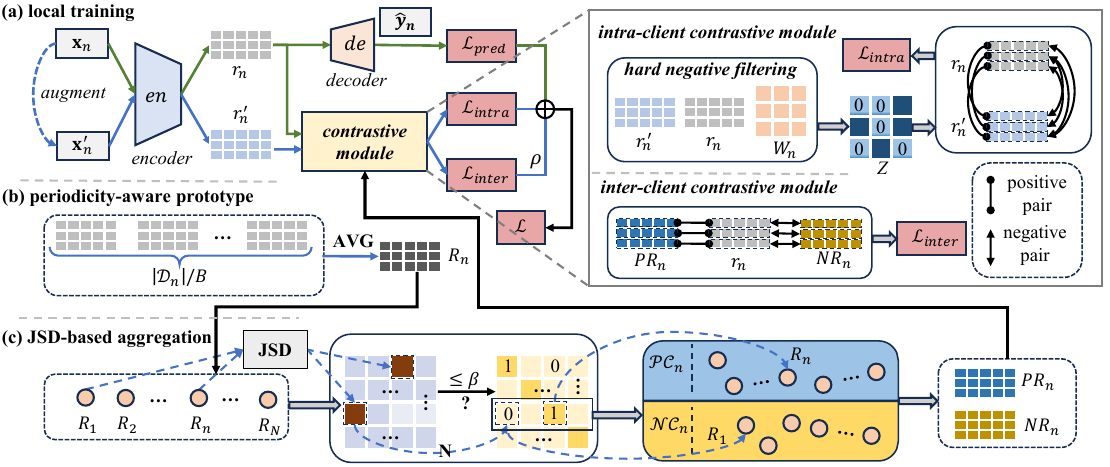}
    \caption{An overview of the proposed FUELS. (a) Each client performs local training by the supplemented inter- and intra-client contrastive loss items for spatio-temporal heterogeneity. (b) The designed periodicity-aware prototype works as the communication carrier. (c) The JSD-based aggregation generates client-customized global prototypes.}
    \label{method}
    \vskip -0.1in
\end{figure*}
We firstly provide an overarching depiction of FUELS, as shown in Figure \ref{method}. In macroscopic view, local training is enhanced by two designed contrastive loss items (Figure \ref{method} (a)). From temporal perspective, we propose a hard negative filtering module for true negative alignment.
From spatial perspective, we employ prototypes as communication carrier (Figure \ref{method} (b)), and generate specific global prototypes (Figure \ref{method} (c)) for spatial heterogeneity evaluation.
We elaborate the technical details in the following parts.

\subsection{Encoder and Decoder} 
In FUELS, we split the prediction model $f$ into two components.
(i) $encoder$ $ en(\delta_n;\textbf{x}_n): \mathbb{R}^{B \times (c+q)} \rightarrow \mathbb{R}^{B \times d_r}$, is parameterized by $\delta_n$ and maps $\textbf{x}_n$ into a latent space with $d_r$ dimensions, where $\left (\textbf{x}_n, \textbf{y}_n \right)$ denotes a batch of training samples from $\mathcal{D}_n$ with batch size $B$.
Actually, $\textbf{x}_n=(\textbf{cv}_n, \textbf{pv}_n)$, where $\textbf{cv}_n \in \mathbb{R}^{B\times c}$ and $\textbf{pv}_n \in \mathbb{R}^{B\times p}$. 
we adopt two widely used Gated Recurrent Unit (GRU) models to evaluate the closeness and periodicity from $\textbf{cv}_n$ and $\textbf{pv}_n$, denoted as closeness-GRU $(\rm{GRU}_c)$ and periodicity-GRU $(\rm{GRU}_p)$ respectively.
${\rm GRU_c}(\delta_n^c;\textbf{cv}_n)$: $\mathbb{R}^{B\times c} \rightarrow \mathbb{R}^{B \times \frac{{d_r}}{2}}$, is parameterized by $\delta_n^c$.
${\rm GRU_p}(\delta_n^p; \textbf{pv}_n)$:
$\mathbb{R}^{B \times q} \rightarrow \mathbb{R}^{B \times \frac{{d_r}}{2}}$, is parameterized by $\delta_n^p$.
We denote the \textbf{representation} of $\textbf{x}_n$ output by the encoder $en$ as $r_n \in \mathbb{R}^{B\times d_r}$, which can be formulated as
\begin{align}
    r_n &= en({\delta_n}; \textbf{x}_n)
    \notag
    \\
    &= { \textbf{concat}}\left[ {\rm GRU_c} ({\delta}_n^c;\textbf{cv}_n) ; {\rm GRU_p} ({\delta}_n^p;\textbf{pv}_n) \right].
\end{align}

(ii) $decoder$ $de(\phi_n ;\cdot): \mathbb{R}^{B \times d_r} \rightarrow \mathbb{R}^{B}$, is parameterized by $\phi_n$ and generates the final predicted result $\hat{\textbf{y}}_n$ from the representation, i.e., $\hat{\textbf{y}}_n =  de(\phi_n; r_n)$.
Similar with many prediction methods, we adopt the fully connected layers as the decoder $de$.

\subsection{Intra-Client Contrastive Task for Temporal Heterogeneity}
In \cite{liu2022contrastive}, hard negatives are filtered out based on traffic closeness but those out of the preset closeness scope are still utilized to form negative pairs and may perturb the latent semantic space. 
Therefore, we propose a hard negative filtering module to align representations by semantic similarity and then design a intra-client contrastive task to maximize the divergence of negative pairs, thereby injecting precise temporal heterogeneity into temporal representations.

Firstly, we adopt the temporal shifting manner in \cite{liu2022contrastive} to generate the augmented dataset for client $n$, which is denoted as ${\mathcal{D}}'_n$.
For $\textbf{x}_{n}$, the corresponding augmented batch in $\mathcal{D}'_n$ is denoted as $\textbf{x}'_{n}$. 
The representation of $\textbf{x}'_{n}$ is denoted as $r'_{n} \in \mathbb{R}^{B \times d_r}$.
Let $r_{n,i}' \in \mathbb{R}^{d_r}$ denote the $i$-th row of $r_{n}'$, i.e., the temporal representation of the $i$-th time stamp of $\textbf{x}'_{n}$.

The procedure of hard negative filtering is formulated as
\begin{align}
    SM &= {\exp \left( {{\rm sim}\left( {r_{n},r_{n}'} \right)/\tau } \right)},\\
    Z &= ReLu \left( SM \odot W_n \right),\\
    \mathcal{TN}_b &= \left\{ i | Z_{b,i} > 0   \right\},
\end{align}
where $SM\in \mathbb{R}^{B\times B}$ denotes the similarity matrix, ${\rm sim}(\cdot , \cdot)$ denotes the cosine similarity and $\tau$ denotes the temperature factor.
$W_n \in \mathbb{R}^{B\times B}$ denotes the learnable filtering matrix and $\odot$ denotes the Hadamard product. 
If $Z_{b,i} = 0 (i\neq b)$, we treat $r_{n,i}'$ as a hard negative of $r_{n,b}$. 
If $Z_{b,i} > 0$, $r_{n,i}'$ is seen as a true negative of $r_{n,b}$ and $i$ is added into $\mathcal{TN}_b$. 
Ideally, $Z_{b,b}(1\le b\le B)$ should be equal to 0 as we also filter out the positive pairs by $W_n$ (detailed results in Section \ref{case}).

By this procedure, representations of different time stamps but similar semantics can also be filtered out. The subsequent contrastive task will repel the representations with diverse semantics so as to validly inject temporal heterogeneity into representations.
Finally, we can obtain the intra-client contrastive loss item as
\begin{equation}
    {\mathcal{L}_{intra}} = \frac{1}{B}\sum\limits_{b = 1}^B { - \log \frac{SM_{b,b}}{{ SM_{b,b} + \sum\nolimits_{i \in \mathcal{TN}_b } {Z_{b,i}} }}} ,
\end{equation}
where $SM_{b,b}$ represents the similarity of positive pair and $Z_{b,i}$ denotes that of negative pair.
Under the constraint of $\mathcal{L}_{intra}$, the prediction model will generate traffic representations with great temporal distinguishability.

\subsection{Inter-Client Contrastive Task for Spatial Heterogeneity}
In this section, we focus on how to preserve the spatial heterogeneity in the communication carrier of FL paradigm and fulfill the aim of assisting local optimization and simultaneously modeling spatial heterogeneity by shared knowledge across clients.

\subsubsection{\textbf{Local Prototype Definition}}
In conventional FL paradigm, model parameters serve as communication carrier, which are not compact and hardly preserve the spatial heterogeneity.
Inspired by prototype learning, where prototypes represent the representations of samples with the same labels and thus can carry the class-specific semantic knowledge, we can extract the client-level prototype to carry the client-specific knowledge.

To achieve this goal, we start with a concatenation-based prototype, which is formulated as
\begin{equation}
    {R_n} = {{\textbf{concat}}}\left[ {{r_n}\left| {{r_n} = en\left( {{\textbf{x}_n}} \right),\forall \left( {{\textbf{x}_n},{\textbf{y}_n}} \right) \in {\mathcal{D}_n}} \right.} \right],
\end{equation}
where $R_n$ denote the local prototype of client $n$. $\mathcal{D}_n$ can be divided into $|\mathcal{D}_n|/B$ batches, with batch size equal to $B$ and hence $R_n\in \mathbb{R}^{|\mathcal{D}_n| \times d_r}$. 
Each client should transmit ${|\mathcal{D}_n| \times d_r}$ parameters to the server at each training round, which potentially incurs more communication overhead than conventional FL methods with $\mathcal{D}_n$ increasing.

Given the periodicity of traffic data, if we set the batch size $B$ equal to the traffic periodicity $p$,  representations of different batches will have similar distribution, as they all carry the traffic knowledge within a periodicity.
Therefore, we can fuse these representations from different batches to obtain the periodicity-aware prototype as
\begin{equation}
    {R_n} = {{\textbf{AVG}}}\left[ {{r_n}\left| {{r_n} = en\left( {{\textbf{x}_n}} \right),\forall \left( {{\textbf{x}_n},{y_n}} \right) \in {\mathcal{D}_n}} \right.} \right],
\end{equation}
where \textbf{AVG} represents the averaging operation.
Therefore, $R_n \in \mathbb{R}^{B\times d_r}$. 
The number of communication parameters keeps stable regardless of local dataset size, and is significantly fewer than the parameter size of local prediction model (see Section \ref{rq1}).
Furthermore, compared with the aforementioned concatenation-based prototype, the periodicity-aware prototype is less affected by traffic noise and thus contains more available client-specific knowledge.
\textit{Therefore, in FUELS, we adopt the periodicity-aware prototype as communication carrier, unless otherwise stated.  }

\subsubsection{\textbf{JSD-based Aggregation}}
For each client, the aggregation process in FUELS can align the homogeneous and heterogeneous semantic prototypes so as to help model the spatial homogeneity and heterogeneity with others in the auxiliary contrastive task.


The diversity of traffic data across clients results in the distribution difference over $R_{n}(1\le n\le N)$, which can be evaluated by JSD at the server. 
Let $JS(R_{n}||R_{m}) \in \mathbb{R}$ denote the JSD value between $R_{n}$ and $R_{m}$.
According to the calculation logic, if the heterogeneity between $R_n$ and $R_m$ is stronger, $JS(R_{n}||R_{m})$ will be higher.
Given a threshold $\beta$, if $JS(R_{n}||R_{m}) \le \beta$, the server considers client $n$ and $m$ share similar knowledge and have homogeneous traffic data.
Then, $R_m$ will be added into $\mathcal{PC}_n$, which denotes the set of positive prototypes for client $n$.
Otherwise, $R_m$ will be put into $\mathcal{NC}_n$, which denotes the set of negative prototypes for client $n$.
Finally, the server performs customized aggregation for client $n$ as
\begin{align}
P{R_{n}} &= \frac{1}{{\left| {{\mathcal{PC}_{n}}} \right|}}\sum\limits_{R_m \in {\mathcal{PC}_{n}}} {{R_{m}}} ,
\\
N{R_n} &= \frac{1}{{\left| {{\mathcal{NC}_n}} \right|}}\sum\limits_{R_m \in  {\mathcal{NC}_n}}^N {{R_m}},
\end{align}
where $PR_{n} \in \mathbb{R}^{B\times d_r}$ and $NR_{n} \in \mathbb{R}^{B\times d_r}$ represent the global positive and negative prototypes for client $n$. 

It is worth noting that the computation complexity of JSD values is $\mathcal{O}\left(\frac{{{N^2} - N }}{2}\right)$, which can hardly apply to forecasting tasks with masses of clients.
The server can randomly select several clients as participants each round, with the selection ratio is $\alpha$. Therefore, the JSD values of those clients which are not selected can be reused. 
The computation complexity can be reduced by $\mathcal{O}\left(\frac{{{N^2}{{(1 - \alpha )}^2} - N(1 - \alpha )}}{2}\right)$. The reduction enlarges significantly with $N$ increasing.

\subsubsection{\textbf{Inter-Client Contrastive Loss}}
After aggregation, the customized global positive and negative prototypes are distributed to the corresponding clients. 
Each client can further enforce the spatial discrimination of its local prediction model by approaching the positive prototype and keeping away from the negative prototype in the training process.

We define the inter-client contrastive loss item as
\begin{align}
{\mathcal{L}_{inter}} &= \frac{1}{B}\sum\limits_{b = 1}^B { - \log \frac{{los{s_{pos}}}}{{los{s_{pos}} + los{s_{neg}}}}}, \\
los{s_{pos}} &= \exp \left( {{\rm sim}\left( {r_{n,b},PR_{n,b}} \right)/\tau } \right),\\
los{s_{neg}} &= \exp \left( {{\rm sim}\left( {r_{n,b},NR_{n,b}} \right)/\tau } \right),
\end{align}

where $PR_{n,b}\in \mathbb{R}^{d_r}$ and $NR_{n,b}\in \mathbb{R}^{d_r}$ denote the $b$-th row of $PR_{n}$ and $NR_{n}$ respectively.
Under the constraint of $\mathcal{L}_{inter}$, the prediction model of the $n$-th client can be empowered with great spatial personalization.


\subsection{\textbf{Local Training and Inference}}
In the training process, client $n$ inputs the representation $r_n$ from the encoder $en$ to the decoder $de$ for generating the predicted values.
Then, it calculates the prediction loss as
\begin{equation}
    \mathcal{L}_{pred} \left(\delta_n ;\phi_n;\textbf{y}_n,\hat{\textbf{y}}_n  \right) = \frac{1}{B} \left\| \textbf{y}_n - \hat{\textbf{y}}_n \right\| ^2 .
\end{equation}
Therefore, the local loss function $\mathcal{L}$ in Eq. (1) is defined as a combination of $\mathcal{L}_{intra}$, $\mathcal{L}_{inter}$, and $\mathcal{L}_{pred}$, which is formulated as
\begin{align}
    \mathcal{L}&\left( { \delta_n,\phi_n, W_n ; \textbf{x}_n,\textbf{y}_n, \textbf{x}_n', PR_n,NR_n} \right) 
    \notag
    \\
    = &{\mathcal{L}_{pred}}\left(  \delta_n, \phi_n; \textbf{x}_n,\hat{\textbf{y}}_n \right)
    \notag
    \\
    &+ {\mathcal{L}_{ intra}}\left( { \delta_n, W_n; \textbf{x}_n , \textbf{x}_n'} \right) 
    \notag
    \\
    &+ \rho {\mathcal{L}_{ inter}}\left( { \delta_n; \textbf{x}_n, PR_n, NR_n} \right),
\end{align}
where $\rho$ denotes the additive weight of $\mathcal{L}_{inter}$. Then, client $n$ performs gradient descent to update local parameters. The training process of FUELS is elaborated in Appendix \ref{process}.
After the training process, the local encoders are of great personalization to generate spatio-temporal heterogeneous representations.
Therefore, in the inference process, each client just needs to input the representations into the decoder to obtain the prediction results.

\section{Algorithm Analysis}
\subsection{Generalization Analysis}
We provide insights into the generalization bound of FUELS. The detailed proof and derivations are presented in Appendix \ref{proof}.
For ease of notation, we use the shorthand $\mathcal{L}\left( w_n \right) := \mathcal{L}\left( { \delta_n,\phi_n, W_n ; \textbf{x}_n,\textbf{y}_n, \textbf{x}_n', PR_n,NR_n} \right) $. 
\begin{assumption}(\textit{Bounded Maximum})
The Loss function $\mathcal{L}(\cdot)$ has an upper bound, i.e., $\mathop {\max } \mathcal{L}(\cdot) \le {\rm C}, {\rm C} < \infty .$
\end{assumption}
\begin{theorem}(\textit{Generalization Bounded})
\label{theorem1}
Let $w_n^*, n\in [1, N]$ denote the optimal model parameters for client $n$ by FUELS. 
Denote the prediction model $f$ as a hypothesis from $\mathcal{F}$ and $d$ as the VC-dimension of $\mathcal{F}$. With the probability at least 1-$\kappa$:
\begin{align}
&\mathop {\max }\limits_{\left( {{w_1},\cdots,{w_N}} \right)} \left[ {\sum\limits_{n = 1}^N {\frac{{\left| {{\mathcal{D}_n}} \right|}}{D}} \mathcal{L}\left( {{w_n} } \right) - \sum\limits_{n = 1}^N {\frac{{\left| {{\mathcal{D}_n}} \right|}}{D}} \mathcal{L}\left( {w_n^*} \right)} \right]
\notag\\
&\le \sqrt {\frac{{2d}}{D}\log \frac{{eD}}{d}}  + \sqrt {\frac{{{{\rm C}^2}{D^2}}}{2}\log \frac{1}{\kappa }},
\end{align}
\end{theorem}
where $e$ denotes the Euler's number. Theorem \ref{theorem1} indicates that the performance gap between FUELS and the optimal parameters is related to the VC-dimension of $\mathcal{F}$, which can be narrowed by carefully-selected prediction networks.

\subsection{Convergence Analysis}
\begin{assumption}(\textit{Bounded Expectation of Gradients})
\label{G}
The expectation of gradient of loss function $\mathcal{L}(\cdot)$ is uniformly bounded, i.e., ${\rm \mathbb{E}}(||\nabla {\mathcal{L}}( \cdot )|{|}) \le {G}.$
\end{assumption}
\begin{assumption}(\textit{Lipschitz Smooth})
\label{l-smooth}
The loss function $\mathcal{L}$ is $L_1$-smooth, i.e., $\mathcal{L}(w) - \mathcal{L}(w') \le \left\langle {\nabla \mathcal{L}(w'), w-w'} \right\rangle  + L_1||w-w'|{|^2}, \forall w, w', \exists L_1 > 0.$
\end{assumption}
\begin{assumption}(\textit{Lipschitz Continuity})
\label{continuous}
Suppose $h:{\mathcal{A}_1} \times$
${\mathcal{A}_2} \times  \cdots  \to {\mathcal{A}_T}$ is Lipschitz Continuous in $\mathcal{A}_j$, i.e., $\exists L_h, \forall$
$a_j, \hat{a}_j \in \mathcal{A}_j,  \left\| {h({a_1}, \cdots ,{a_j}, \cdots ) - h({a_1}, \cdots ,{{\hat a}_j}, \cdots )} \right\| \le {L_h}\left\| {{a_j} - {{\hat a}_j}} \right\|$.
\end{assumption}
Let $w_{n,t}^i$ denote the local model parameters of client $n$ at the $i$-th iteration in the $t$-th FL round, where $0\le t\le T-1$ and $0\le i\le I-1$. $T$ and $I$ denote the total number of FL rounds and local iterations respectively.
\begin{theorem}(\textit{Convergence Rate})
\label{convergence}
    Let Assumption \ref{G} to \ref{continuous} hold. Let $\mathcal{L}\left( {w_{n,0}^0} \right) - \mathcal{L}\left( {w_n^*} \right) = \Lambda $.
    If clients adopt stochastic gradient descent method to optimize local prediction models with the learning rate equal to $\eta$, for any client, given $\xi >0$, after
    \begin{equation}
        \label{T}
        T = \frac{\Lambda }{{\xi I(\eta  - {L_1}{\eta ^2}) - \rho \eta {L_h}N\alpha {I^2}G}}
    \end{equation}
    FL rounds, we can obtain 
    \begin{equation}
        \frac{1}{{TI}}\sum\limits_{t = 0}^{T - 1} {\sum\limits_{i = 0}^{I - 1} {\mathbb{E}\left[ {{{\left\| {\nabla \mathcal{L}\left( {w_{n,t}^i} \right)} \right\|}^2}} \right]} }  < \xi 
    \end{equation}
    with 
    \begin{equation}
        \label{eta}
        \eta < \frac{{\xi  - \rho {L_h}N\alpha IG}}{{{L_1}\xi }}.
    \end{equation}
\end{theorem}
Based on the above commonly-used assumptions in FL works \cite{kairouz2021advances}, Theorem \ref{convergence} (the detailed proof is presented in Appendix \ref{proof2}) provides the convergence rate of FUELS. 
By adopting the learning rate $\eta$ computed via Eq. (\ref{eta}), after $T$ FL rounds (calculating $T$ via Eq. (\ref{T})), the expectation of model updates will not exceed the given arbitrary value $\xi$.


\begin{table*}[htbp]
    \centering
  \caption{Performance comparison of different methods on three benchmark datasets. 
  ``\# of Comm Params'' refers to the averaged number of parameters each client sends to the server per round. ``Average'' represents the averaging performance on three datasets.
  }
  \vskip 0.1in
    \begin{tabular}{cccccccccc}
    \toprule
    Dataset & \multicolumn{2}{c}{SMS} & \multicolumn{2}{c}{Call} & \multicolumn{2}{c}{Net} & \multicolumn{2}{c}{Average} & \multicolumn{1}{c}{\multirow{2}[4]{*}[+1.5ex]{\makecell[c]{\# of Comm \\ Params} }} \\
    Metric & MSE   & MAE   & MSE   & MAE   & MSE   & MAE   & MSE   & MAE   &  \\
    \midrule
    Solo  & 1.884  & 0.604  & 0.361  & 0.294  & 2.423  & 0.654  & 1.556  & 0.517  & - \\
    FedAvg & 1.452  & 0.533  & 0.393  & 0.300  & 2.649  & 0.638  & 1.498  & 0.490  & 100737 \\
    FedProx & 1.495  & 0.542  & 0.394  & 0.300  & 2.528  & 0.629  & 1.472  & 0.490  & 100737 \\
    FedRep & 1.551  & 0.557  & 0.372  & 0.299  & 2.288  & 0.625  & 1.404  & 0.494  & 100737 \\
    PerFedAvg & 1.253  & 0.553  & 0.392  & 0.336  & 1.107  & 0.543  & 0.917  & 0.478  & 100737 \\
    pFedMe & 1.250  & 0.549  & 0.409  & 0.335  & 1.184  & 0.546  & 0.948  & 0.476  & 100737 \\
    FedDA & 1.937  & 0.699  & 1.055  & 0.430  & 3.589  & 0.849  & 2.193  & 0.659  & 100737 \\
    \rowcolor{blue!40} FUELS & 1.249  & 0.541  & 0.353  & 0.311  & 0.880  & 0.488  & 0.827  & 0.446  & 6144 \\
    \bottomrule
    \end{tabular}%
  \label{performance}
\vskip -0.1in
\end{table*}%

\subsection{Communication and Computation Complexity}
The number of parameters each client uploads is $\mathcal{O}(B{d_r})$, which is much smaller than the number of model parameters (detailed numerical results in Table \ref{performance}).
The computation cost at client in training FL round can be summarized as $\mathcal{O}(I(3FE+2FF+2FD))$, where $FE$, $FF$, and $FD$ represent the computation cost of encoder $en$, $W_n$, and decoder $de$ respectively for a batch in the forward propagation. 
Given the light weight of $W_n$, most of additional computation cost results from encoding the augmented data.
In Appendix \ref{analysis}, we further provide detailed comparison of computation and communication complexity among FUELS and state-of-the-art methods.

\section{Experiments}
In this section, we conduct extensive experiments to validate the effectiveness and efficiency of our proposed FUELS to answer the following research questions: \textbf{RQ1:} Can FUELS achieve dominant prediction performance with the baselines in a communication-efficient manner? \textbf{RQ2:} Do the periodicity-aware prototypes outperform the concatenation-based prototypes? \textbf{RQ3:} How do different loss items in FUELS affect the prediction performance? 
\textbf{RQ4:} Does the dynamic hard negative filtering really work?
\textbf{RQ5:} How do hyperparameters affect prediction performance?

\subsection{Experimental Settings}
\label{perfom-set}
We evaluate our proposed method on the following three benchmark datasets, i.e., short message services (SMS), voice calls (Call), and Internet services (Net) \cite{barlacchi2015multi}. 
The performance of FUELS is compared with 6 popular FL methods, including FedAvg \cite{mcmahan2017communication}, FedProx \cite{li2020federated}, FedRep \cite{collins2021exploiting}, PerFedAvg \cite{fallah2020personalized}, pFedMe \cite{DBLP:conf/nips/DinhTN20}, FedDA \cite{zhang2021dual}, and Solo. 
We adopt two widely used metrics to evaluate the prediction performance, i.e, Mean Squared Error (MSE) and Mean Absolute Error (MAE).
Due to space limitation, detailed introduction of experimental settings is presented in Appendix \ref{exprimnt}.

\subsection{Performance Comparison (RQ1)}
\label{rq1}

The evaluation results of the baselines and our proposed method on three datasets are presented in Table \ref{performance}.
We have the key observations that FUELS shows promising performance by consistently outperforming all the baselines on the three datasets. 
Moreover, the communication cost of clients in FUELS is significantly lower than that in the baselines, with approximate 94\% reduction in communication parameters per client at each round.
We attribute the superiority to the designed contrastive tasks and novel communication carrier.

\textbf{Effectiveness.} Figure \ref{cdfs} shows the predicted values of four methods and the cumulative distribution functions (CDFs) of MSEs over all clients.
We observe that the predicted values of FedRep, PerFedAvg, pFedMe and FUELS have similar prediction performance at the smooth traffic sequences.
However, FUELS can generate more accurate results in fluctuating traffic sequences on all three datasets.
The area under the CDF curve of FUELS is larger than those of the other three methods, which indicates that clients' prediction MSEs in FUELS distribute around lower values. Taking the Net dataset for example, 87\% of clients have prediction MSEs lower than 1.5, while the cases for FedRep, pFedMe and PerFedAvg are 72\%, 80\% and 81\% respectively.

\textbf{Efficiency.} Figure \ref{mse-com} further shows the training MSE versus the communication amounts over all clients.
We observe that given the same MSE, the communication amounts in pFedMe and FedRep are significantly higher than those in FUELS on both datasets.
FUELS can yield superior prediction performance with the baselines and simultaneously reduce the communication cost to a great extent, which indicates the efficiency of FUELS.

\begin{figure}[!t]
\vskip 0.1in
    \centering
    \includegraphics[width=\linewidth]{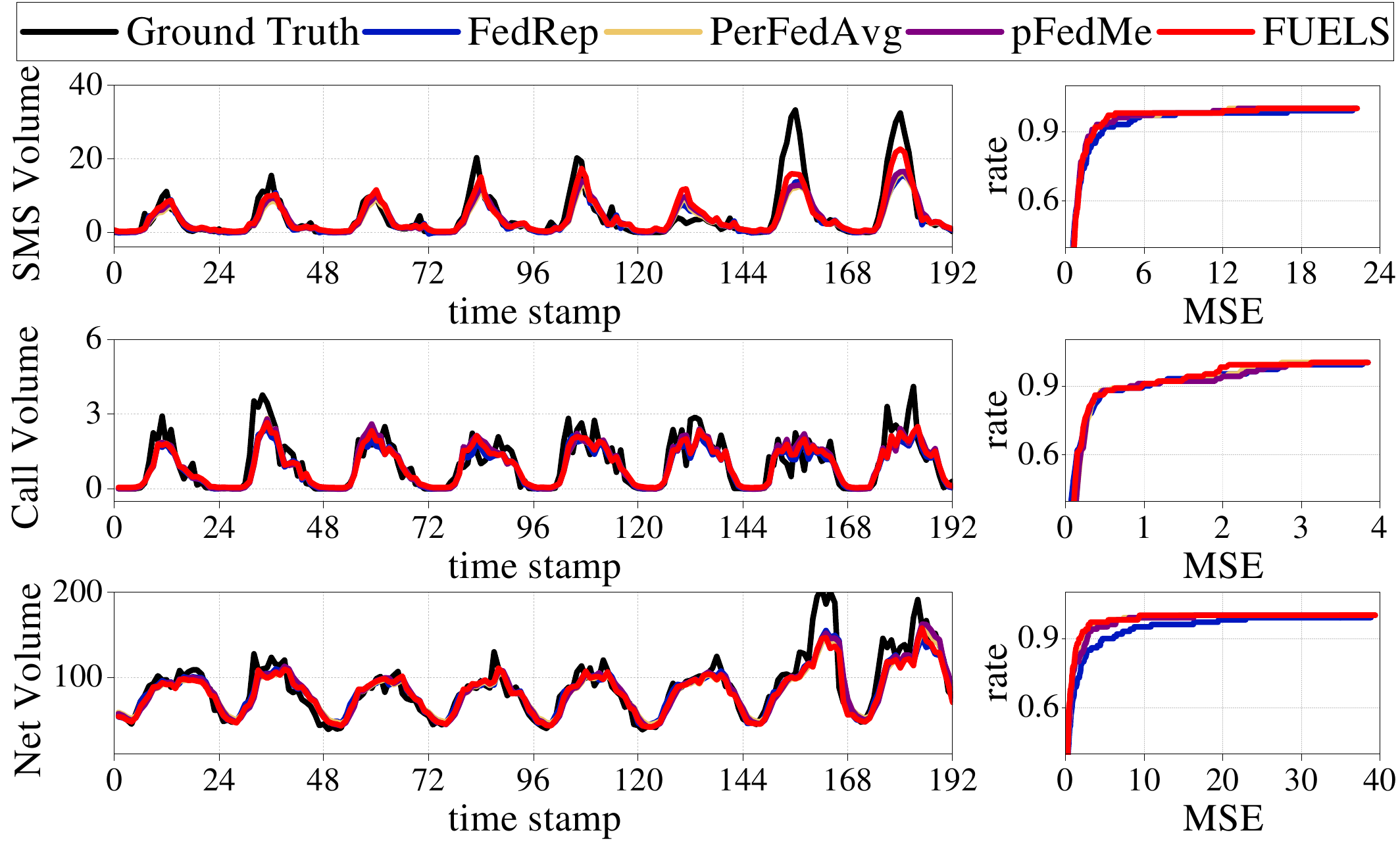}
    \caption{Comparison of different methods in terms of prediction values and CDFs of MSE.}
    \label{cdfs}
    \vskip -0.1in
\end{figure} 
\begin{figure}[ht]
\vskip 0.1in
    \centering
    \includegraphics[width=\linewidth]{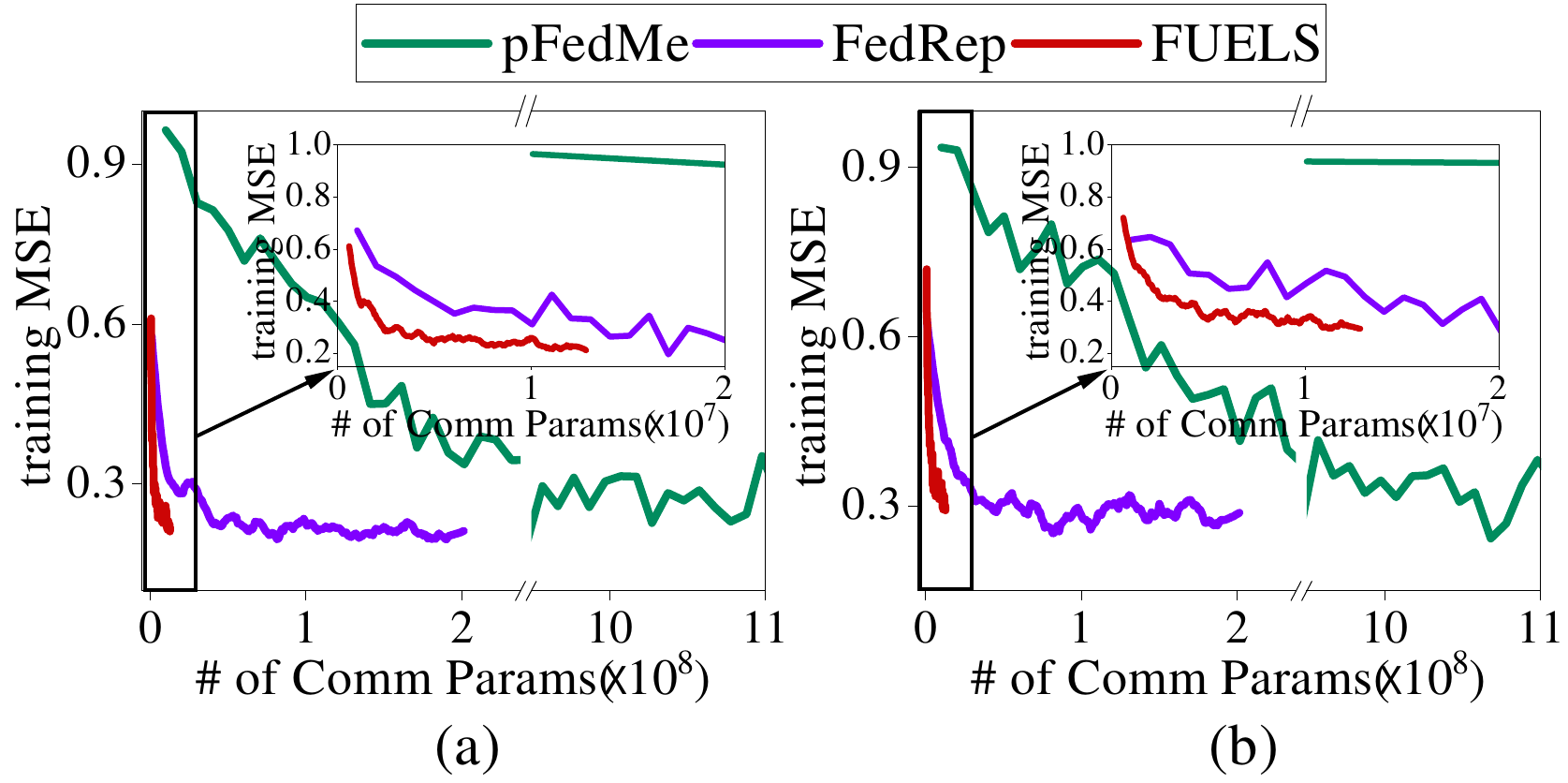}
    \caption{Training MSE versus communication amounts on (a) SMS and (b) Net datasets. }
    \label{mse-com}
    \vskip -0.1in
\end{figure}

\subsection{Ablation Study (RQ2, RQ3 and RQ4)}
we compare FUELS with the following 4 variants. \textbf{(1)} \textbf{w/o inter} \textbf{(for RQ3)}: Only the intra-client contrastive loss item is adopted; \textbf{(2)} \textbf{w/o intra} \textbf{(for RQ3)}: Only the inter-client contrastive loss item is adopted; \textbf{(3)} \textbf{w/o p-aware} \textbf{(for RQ2)}: Concatenation-based prototypes are adopted instead of {\underline p}eriodicity-{\underline {aware}} prototypes; 
\textbf{(4)} \textbf{w/o \textit{W}} \textbf{(for RQ4)}: Dynamic hard negative filtering is omitted. The prediction performance of FUELS and the above variants on Net dataset is shown in Figure \ref{ablation}. Full results are presented in Table \ref{ab_full}.


We have the following key observations.
\textbf{(1)} The increased prediction errors of w/o inter and w/o intra compared with FUELS indicate that inter- and intra-client contrastive loss items can benefit the local training from different perspectives. 
\textbf{(2)} With concatenation-based prototypes, the prediction performance may be impacted by the noise in representations and ``\# of Comm Params'' will significantly increase to 227328.
\textbf{(3)} The higher error in w/o $W$ indicates that the proposed dynamic filtering module can effectively filter out hard negatives and keep the semantic structure consistent.

\begin{figure}[!ht]
\vskip 0.1in
    \centering
    \includegraphics[width=.9\linewidth]{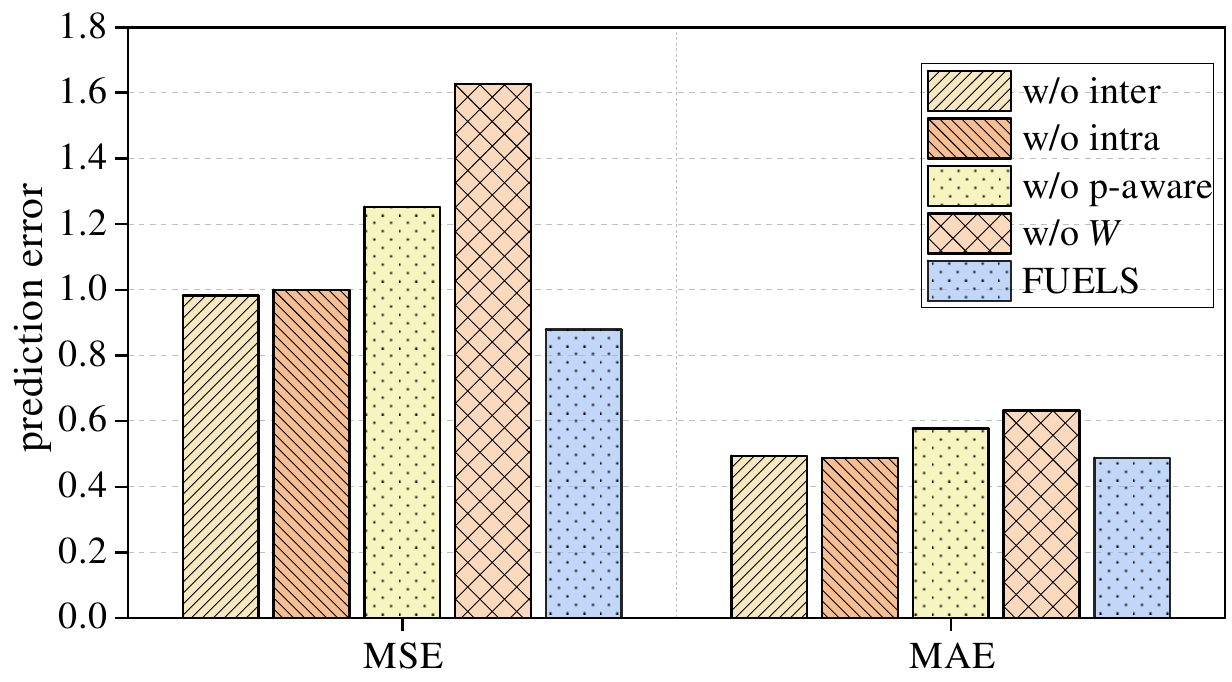}
    \caption{Performance comparison of FUELS and four variants on Net dataset.}
    \label{ablation}
    \vskip -0.1in
\end{figure}

\subsection{Hyperparameter Investigation (RQ5)}
\textbf{We investigate the effects of different settings of $\beta$, $\tau$, and $\rho$ on prediction performance}, which is shown in Figure \ref{param}.
\textbf{(1)} $\beta$ affects how the server divides the positive and negative sets for clients. Too larger or lower values of $\beta$ will make the server fail to strike the balance between guaranteeing the spatial heterogeneity and sharing similar knowledge. When $\beta$ is set as the 50-th percentile of all JSD values, FUELS achieves the best performance.
\textbf{(2)} A lower value of $\tau$ will make the contrastive loss pay more attention to punishing the negative samples \cite{wang2021understanding}.
In general, MSE gets higher with the increase of $\tau$ and we finally set $\tau$ as 0.02.
\textbf{(3)} Different settings of $\rho$ determine how deep clients pay attention to the spatial heterogeneity. 
To keep a balance between the temporal and spatial heterogeneity, $\rho$ should be neither too high nor to low. From Figure \ref{param}(c), the best setting of $\rho$ is 5.
\begin{figure}[!ht]
\vskip 0.1in
    \centering
    \includegraphics[width=0.9\linewidth]{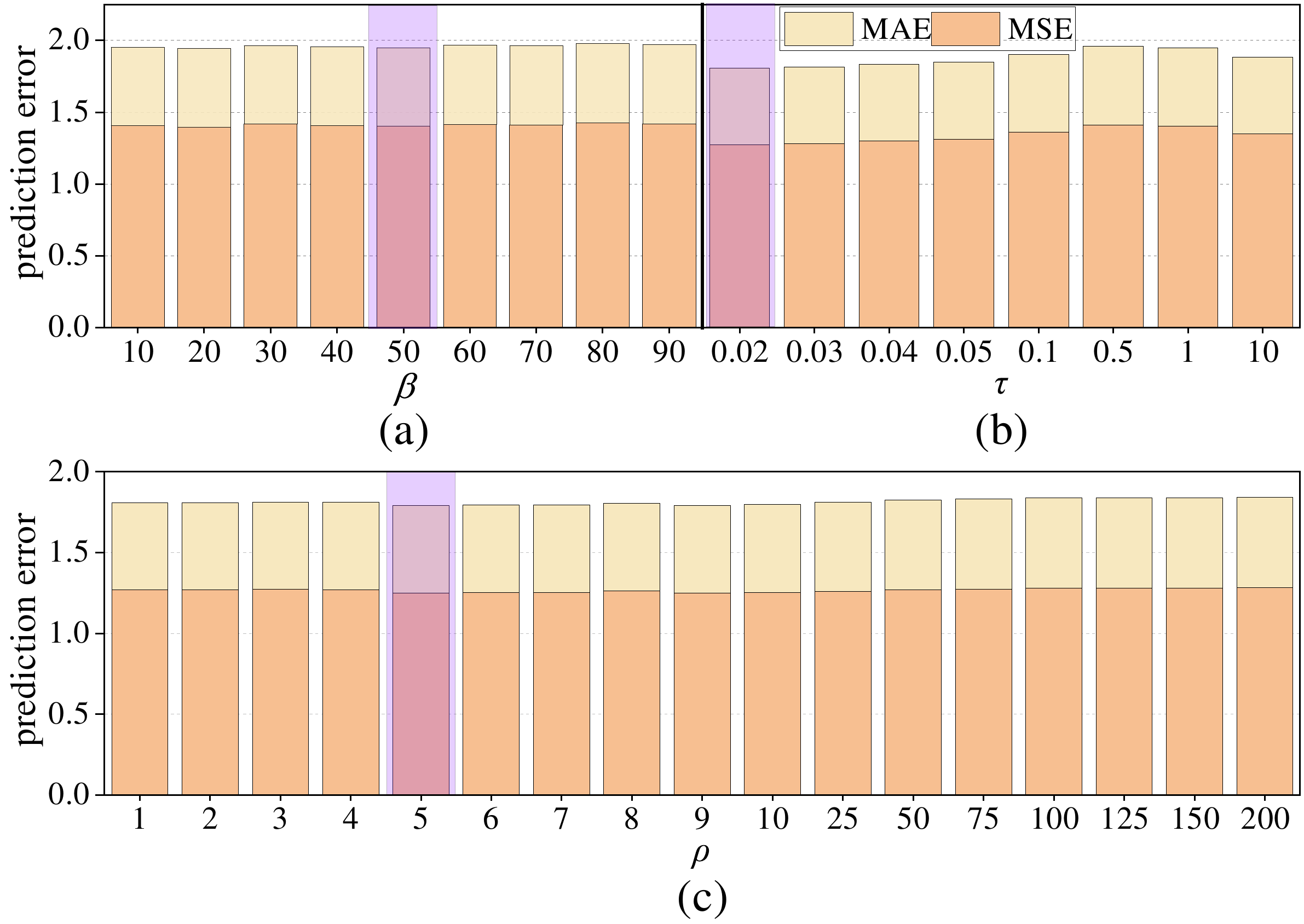}
    \caption{Different settings of (a) $\beta$, (b) $\tau$, and (c) $\rho$ versus prediction errors on  SMS dataset.}
    \label{param}
\vskip -0.1in
\end{figure}
\subsection{Case Study}
\label{case}
\textbf{Relatedness Between Local Prototype and Traffic Data.}
We randomly select several clients to visualize the correlation between prototypes and local datasets. The results are shown in Figure \ref{js}(a) and (b).
We calculate the similarity of local training datasets based on JSD and set the threshold as $0.00065$, which is also the 50-th percentile of these JSD values. We observe that clients' correlation illustrated in Figure \ref{js}(a) is the same as the JSD-based results in Figure \ref{js}(b). Therefore, we claim that the designed local prototypes can effectively express client-specific knowledge.

\textbf{Visualization of Dynamic Hard Negative Filtering.} The illustration of parameters in $W_n$ for a randomly selected client is presented in Figure \ref{filter}. 
If a parameter in $W_n$ is over 0, the corresponding value in $Z$ will be over 0, otherwise equal to 0. 
Almost all the values on the diagonal are less than 0, which represents that positives are filtered out.
Furthermore, the filtering approach can filter out hard negatives within or without closeness scope and simultaneously avoid the error filtering of true negatives within closeness range.

\subsection{Combination with Privacy Mechanisms }
We incorporate FUELS with privacy-preserving mechanisms in case of the privacy leakage in the communicated prototypes.
Similar to differential privacy, \textbf{we add different types of random noise with $Laplace$, $Guassian$, and $exponential$ distribution to the local prototypes in FUELS.} 
The results indicate that there is no significant performance degradation in the performance of FUELS after noise injection.
Details are presented in \ref{sec_privacy}.

\subsection{Additional Experimental Results in Appendix}
Due to space limitation, we provide auxiliary experimental results in Appendix, including \textbf{effect of prototype size} (in Appendix \ref{append-proto}) and \textbf{performance on extra dataset} (Appendix \ref{other}).

\section{Related Work}
\subsection{Federated Learning for Spatio-Temporal Forecasting}
Due to the effectiveness of FL, there have been many works focusing on incorporating FL into spatio-temporal forecasting tasks, e.g., wireless traffic prediction \cite{zhang2021dual,perifanis2023federated,zhang2022efficient}.
These FL methods focus on how to evaluate the decoupled spatial and temporal correlation at the server and at the clients respectively by split learning \cite{meng2021cross}, clustering \cite{liu2020privacy,zhang2021dual}, online learning \cite{liu2023online}, and so on.
However, all of these methods ignore the spatio-temporal heterogeneity and yield prediction results for different time stamps or different clients with the shared parameter space.


\subsection{Personalized Federated Learning}
Personalized federated learning (PFL) tackles with the problem of inference performance decline aroused by statistics heterogeneity across clients.
The existing PFL methods can be divided into 2 categories, i.e., clients training a global model or training personalized local models.
Methods in the first category aim to increase the generalization of the global model by designing client-selection strategy \cite{yang2021federated,li2021fedsae} or adding proximal item to original functions  
\cite{karimireddy2020scaffold,li2021model} 
In the second category, researchers modify the conventional FL procedure by network splitting \cite{arivazhagan2019federated,bui2019federated,liang2020think}, multitask learning \cite{smith2017federated,ghosh2020efficient,xie2021multi}, knowledge distillation \cite{li2019fedmd,zhu2021data,lin2020ensemble} and so on.
However, all of these works fail to solve the decoupled spatio-temporal heterogeneity.

\subsection{Contrastive Learning}
In contrastive learning, embeddings from similar samples are pulled closer and those from different ones are pushed away by constraining the loss function \cite{chen2020simple,he2020momentum}. 
Due to its effectiveness, contrastive learning has been applied to many scenarios, i.e., graph learning \cite{li2022mining,you2020graph,zhu2021graph}, traffic flow prediction \cite{ji2023spatio,yue2022ts2vec,woo2022cost}, \emph{etc}.
Furthermore, some researches have focused on introducing contrastive learning into PFL paradigm mainly for handling statistics heterogeneity \cite{li2021model,tan2022federated,yu2022multimodal,mu2023fedproc}.
However, all of these methods focus on improving the performance on image classification. How to tackle with the spatio-temporal heterogeneity PFL paradigm by contrastive learning remains an open problem.

\section{Conclusion and Future Work}
We propose FUELS for tackling the spatio-temporal heterogeneity. Specifically, we adaptively align the temporal and spatial representations according to semantic similarity for the supplemented intra- and inter-client contrastive tasks to preserve the spatio-temporal heterogeneity in the latent representation space.
Note-worthily, a lightweight but efficient prototype is designed as the client-level representation for carrying client-specific knowledge.
Experimental results demonstrate the effectiveness and efficiency of FUELS.

Due to the ubiquity of spatio-temporal heterogeneity, FUELS may be also applicable to other spatio-temporal tasks.
Since the communication carrier is independent of network structures, FUELS can be built over heterogeneous local prediction models, pre-trained models, or even Large Language Models, which will be explored in future studies.

\section*{Impact Statements}
This paper presents work whose goal is to advance the field of Machine Learning. There are many potential societal consequences of our work, none which we feel must be specifically highlighted here.

\nocite{langley00}

\bibliography{example_paper}
\bibliographystyle{icml2024}

\newpage
\appendix
\onecolumn
\section{Notations}
\begin{table}[htbp]
  \centering
  \caption{Main notations and descriptions in Section 2 and 3.}
  \vskip 0.1in
    \begin{tabular}{cc}
    \toprule
    Notation & Description \\
    \midrule
    $n$ & Client index\\
    $N$ & Total number of clients\\
    $\mathcal{D}_n$ & Training dataset of client $n$ \\
    $D = \sum\nolimits_{n = 1}^N {\left| {{\mathcal{D}_n}} \right|} $ & Total number of training samples of $N$ clients\\
    $(\textbf{x}_n, \textbf{y}_n)$ & A batch of training samples of client $n$\\
    $\delta_n = (\delta_n^c, \delta_n^p)$ & Encoder's parameters of client $n$\\
    $\phi_n$ & Decoder's parameters of client $n$\\
    $W_n$ & Filtering matrix of client $n$\\
    $w_n=(\delta_n, \phi_n, W_n)$ & Local parameters of client $n$\\
    $B$ & Local batch size\\
    $r_n$ & Representation of $\textbf{x}_n$\\
    $d_r$ & Dimension of representation\\
    $R_n$ & Local prototype of client $n$\\
    $\mathcal{PC}_n, \mathcal{NC}_n$ & Positive and negative prototype set of client $n$\\
    $PR_n, NR_n$ & Global positive and negative prototype of client $n$\\
    $\beta$ & JSD threshold\\
    $\alpha$ & Selection ratio\\
    $\tau$ & Temperature factor\\
    $\rho$ & Additive weight of $\mathcal{L}_{inter}$\\
    $\mathcal{L}(\cdot)$ & Loss function\\
    $\mathcal{L}_{pred}, \mathcal{L}_{intra}, \mathcal{L}_{inter}$ & Loss items for prediction, intra-client contrastive task, and inter-client contrastive task\\
    \bottomrule
    \end{tabular}%
  \label{notations}%
  \vskip -0.1in
\end{table}%

\section{Detailed Procedure of FUELS}
\label{process}
The training process of FUELS is presented in Algorithm \ref{ag1}. We elaborate the procedure as follows.
\begin{itemize}[leftmargin=*]
    \item \textbf{Initialization:} At the initial round, all the client participating in FUELS should provide their local prototypes. Similarly, for the clients who want to join during the training stage, they should also provide local prototypes firstly.
    \item \textbf{Aggregation:} After the server receiving local prototypes from participating clients, it will update $\mathcal{R}$, perform JSD-based aggregation to generate client-customized global prototypes for the next round.
    \item \textbf{Local Training:} Each selected client calculates the representations of the raw and augmented data, computes three loss items gradually and then optimizes the local prediction model. After all batches, it yields the local prototype and uploads to the server.
\end{itemize}

\begin{algorithm}[!ht]
	\caption{Traning process of FUELS}
	\label{ag1}
	\LinesNumbered
	\KwIn{$\mathcal{D}_n$, $\mathcal{D}'_n$, $n=1,\cdots,N$, $\rho$, $\beta$, $\tau$, $\alpha$.}
	\SetKwFunction{Fexecute}{ClientExecute}
	
	\textsc{\textbf{ServerExecute:}}
 
         Initialize $\left\{ {P{R_n}} \right\}_{n = 1}^N$ and $\left\{ {N{R_n}} \right\}_{n = 1}^N$.

         Initialize the set of local prototypes $\mathcal{R} = \emptyset$.
         
	\For{$t = 1, 2, \cdots, T$}{
            
		\If{$t=1$}{
                \For {each client $n$ \textbf{in parallel}}{
                    $R_n \gets$\Fexecute$(n, PR_n, NR_n)$
                    
                    $\mathcal{R} \gets \mathcal{R} + \{R_n\}$
                
		      }
            }
            \Else{
                Randomly select $N\alpha$ clients.
                
                \For {each selected client $n$ \textbf{in parallel}}{
                    $R_n \gets$\Fexecute$(n, PR_n, NR_n)$

                    Update $R_n$ in $\mathcal{R}$.   
		      }  
            }
            
            Yield $\mathcal{PC}_n$ and $\mathcal{RC}_n, n\in [N]$ based on JSD values.

            $P{R_n} = \frac{1}{{\left| {{\cal P}{{\cal C}_n}} \right|}}\sum\nolimits_{{R_m} \in {\cal P}{{\cal C}_n}} {{R_m}} $, $N{R_n} = \frac{1}{{\left| {{\cal N}{{\cal C}_n}} \right|}}\sum\nolimits_{{R_m} \in {\cal N}{{\cal C}_n}} {{R_m}} \forall n \in [N].$
	}
	\SetKwProg{Fn}{Function}{:}{}
	\Fn{\Fexecute{$n$, $PR_n$, $NR_n$}}{
            \For {each epoch}{
                Initialize the set of representations $\mathcal{R}_n=\emptyset$.
                
                \For{$(\textbf{x}_n,\textbf{y}_n)$ in $\mathcal{D}_n$, $(\textbf{x}'_n,\textbf{y}'_n)$ in $\mathcal{D}'_n$ }{
                    $r_n \gets en(\delta_n;\textbf{x}_n), r'_n\gets en(\delta_n;\textbf{x}'_n)$

                    $\mathcal{R}_n \gets \mathcal{R}_n + \{r_n\}$
                    
                    Calculate $\mathcal{L}_{intra}$, $\mathcal{L}_{inter}$, and $\mathcal{L}_{pred}$ via Eq. (6), (11), and (14).
                    
                    $\mathcal{L} \gets \mathcal{L}_{intra}+ \mathcal{L}_{inter} + \rho \mathcal{L}_{pred} $
                    
                    Update $\delta_n$ and $\phi_n$ via gradient descent.
                }
                $R_n \gets \textbf{AVG}(\mathcal{R}_n)$.
            }
            \KwRet $R_n$
	}
\end{algorithm}

\section{Analysis}
\subsection{Proof of Generalzation Bound}
\label{proof}
\begin{assumption}(\textit{Bounded Maximum})
\label{maximum}
The Loss function $\mathcal{L}(\cdot)$ has an upper bound, i.e., $\mathop {\max } \mathcal{L}(\cdot) \le {\rm C}, {\rm C} < \infty .$
\end{assumption}
\begin{theorem}(\textit{Generalization Bounded})
\label{theorem}
Let $w_n^*, n\in [1, N]$ denote the optimal model parameters for client $n$ by Algorithm \ref{ag1}. 
Denote the prediction model $f$ as a hypothesis from $\mathcal{F}$ and $d$ as the VC-dimension of $\mathcal{F}$. With the probability at least 1-$\kappa$:
\begin{align}
&\mathop {\max }\limits_{\left( {{w_1},\cdots,{w_N}} \right)} \left[ {\sum\limits_{n = 1}^N {\frac{{\left| {{\mathcal{D}_n}} \right|}}{D}} \mathcal{L}\left( {{w_n} } \right) - \sum\limits_{n = 1}^N {\frac{{\left| {{\mathcal{D}_n}} \right|}}{D}} \mathcal{L}\left( {w_n^*} \right)} \right]
\notag\\
&\le \sqrt {\frac{{2d}}{D}\log \frac{{eD}}{d}}  + \sqrt {\frac{{{{\rm C}^2}{D^2}}}{2}\log \frac{1}{\kappa }}\notag,
\end{align}
\end{theorem}
\begin{proof} 
We denote $g(w_1, \cdots, w_N) = \mathop {\max }\limits_{\left( {{w_1},\cdots,{w_N}} \right)} \left[ {\sum\limits_{n = 1}^N {\frac{{\left| {{\mathcal{D}_n}} \right|}}{D}} \mathcal{L}\left( {{w_n}} \right) - \sum\limits_{n = 1}^N {\frac{{\left| {{\mathcal{D}_n}} \right|}}{D}} \mathcal{L}\left( {w_n^*} \right)} \right]$ and can obtain that
\begin{align}
& \left| {g({w_1},\cdots,{w_n},\cdots, {w_N}) - g({w_1},\cdots,{{w}_n'},\cdots,{w_N})} \right|
\notag\\
&= \left| {\mathop {\max }\limits_{\left( {{w_1},\cdots,{w_N}} \right)} \left[ {\sum\limits_{n = 1}^N {\frac{{\left| {{\mathcal{D}_n}} \right|}}{D}} \mathcal{L}\left( {{w_n}} \right) - \sum\limits_{n = 1}^N {\frac{{\left| {{\mathcal{D}_n}} \right|}}{D}} \mathcal{L}\left( {w_n^*} \right)} \right] - \mathop {\max }\limits_{\left( {{w_1},\cdots, {{w}_n'},\cdots,{w_N}} \right)} \left[ {\sum\limits_{n = 1}^N {\frac{{\left| {{\mathcal{D}_n}} \right|}}{D}} \mathcal{L}\left( {{{w}_n'}} \right) - \sum\limits_{n = 1}^N {\frac{{\left| {{\mathcal{D}_n}} \right|}}{D}} \mathcal{L}\left( {w_n^*} \right)} \right]} \right|
\notag\\
&= \left| {\mathop {\max }\limits_{{w_n}} \left[ {\frac{{\left| {{\mathcal{D}_n}} \right|}}{D}\left( {\mathcal{L}\left( {{w_n}} \right) - \mathcal{L}\left( {w_n^*} \right)} \right)} \right] - \mathop {\max }\limits_{{{w}_n'}} \left[ {\frac{{\left| {{\mathcal{D}_n}} \right|}}{D}\left( {\mathcal{L}\left( {{{w}_n'}} \right) - \mathcal{L}\left( {w_n^*} \right)} \right)} \right]} \right|
\notag\\
&= \frac{{\left| {{\mathcal{D}_n}} \right|}}{D}\left| {\mathop {\max }\limits_{{w_n}} \left[ {\mathcal{L}\left( {{w_n}} \right) - \mathcal{L}\left( {w_n^*} \right)} \right] - \mathop {\max }\limits_{{{w}_n'}} \left[ {\mathcal{L}\left( {{{w}_n'}} \right) - \mathcal{L}\left( {w_n^*} \right)} \right]} \right|
\notag\\
&= \frac{{\left| {{\mathcal{D}_n}} \right|}}{D}\left| {\mathop {\max }\limits_{{w_n},{{w}_n'}} \left[ {\mathcal{L}\left( {{w_n}} \right) - \mathcal{L}\left( {{{w}_n'}} \right) + \mathcal{L}\left( {{{w}_n'}} \right) - \mathcal{L}\left( {w_n^*} \right)} \right] - \mathop {\max }\limits_{{{w}_n'}} \left[ {\mathcal{L}\left( {{{w}_n'}} \right) - \mathcal{L}\left( {w_n^*} \right)} \right]} \right|
\notag\\
&\le \frac{{\left| {{\mathcal{D}_n}} \right|}}{D}\left| {\mathop {\max }\limits_{{w_n},{{w}_n'}} \left[ {\mathcal{L}\left( {{w_n}} \right) - \mathcal{L}\left( {{{w}_n'}} \right)} \right]} \right|
\notag\\
&\mathop  \le \limits^{(a)} \frac{{C{{\left| {{D_n}} \right|}^2}}}{D}
\end{align}
(a) follows Assumption \ref{maximum}.
Therefore, $g$ is difference bounded. From the McDiarmid's inequality, we have
\begin{equation}
    \label{mi}
    \mathbb{P}\left[ {g\left( {{w_1},\cdots, {w_N}} \right) - \mathbb{E}\left[ {g\left( {{w_1},\cdots,{w_N}} \right)} \right] \ge \varepsilon } \right] \le \exp ( - \frac{{2{\varepsilon ^2}}}{{\sum\nolimits_{n = 1}^N {c_n^2} }}),
\end{equation}
where $c_n=\frac{{{\rm C}{{\left| {{\mathcal{D}_n}} \right|}^2}}}{D}$. (\ref{mi}) can be transformed into 
\begin{equation}
    \mathbb{P}\left[ {g\left( {{w_1},\cdots,{w_N}} \right) - \mathbb{E}\left[ {g\left( {{w_1},\cdots,{w_N}} \right)} \right] \le \varepsilon } \right] \ge 1- \exp ( - \frac{{2{\varepsilon ^2}}}{{\sum\nolimits_{n = 1}^N {c_n^2} }}),
\end{equation}
which indicates that with probability at least $1- \exp ( - \frac{{2{\varepsilon ^2}}}{{\sum\nolimits_{n = 1}^N {c_n^2} }})$, 
\begin{equation}
    \label{mi2}
    g\left( {{w_1},\cdots,{w_N}} \right) - \mathbb{E}\left[ {g\left( {{w_1},\cdots,{w_N}} \right)} \right] \le \varepsilon 
\end{equation}

Let $\kappa = \exp ( - \frac{{2{\varepsilon ^2}}}{{\sum\nolimits_{n = 1}^N {c_n^2} }})$, and 
\begin{align}
\varepsilon &= \sqrt {\frac{{\sum\nolimits_{n = 1}^N {c_n^2} }}{2}\ln \frac{1}{\kappa }} 
\notag
\\
&= \sqrt {\frac{{\sum\nolimits_{n = 1}^N {{{\left( {\frac{{{\rm C}{{\left| {{\mathcal{D}_n}} \right|}^2}}}{D}} \right)}^2}} }}{2}\log \frac{1}{\kappa }} 
\notag\\
&\le \sqrt {\frac{{\frac{{{{\rm C}^2}}}{{{D^2}}}{{\left( {\sum\nolimits_{n = 1}^N {{{\left| {{\mathcal{D}_n}} \right|}^2}} } \right)}^2}}}{2}\log \frac{1}{\kappa }} 
\notag\\
&\le \sqrt {\frac{{\frac{{{{\rm C}^2}}}{{{D^2}}}{{\left( {\sum\nolimits_{n = 1}^N {\left| {{\mathcal{D}_n}} \right|} } \right)}^4}}}{2}\log \frac{1}{\kappa }} 
\notag\\
&= \sqrt {\frac{{{{\rm C}^2}{D^2}}}{2}\log \frac{1}{\kappa }} 
\end{align}

Therefore, (\ref{mi2}) can be transformed into
\begin{align}
&\mathop {\max }\limits_{\left( {{w_1},\cdots,{w_N}} \right)} \left[ {\sum\limits_{n = 1}^N {\frac{{\left| {{\mathcal{D}_n}} \right|}}{D}} \mathcal{L}\left( {{w_n}} \right) - \sum\limits_{n = 1}^N {\frac{{\left| {{\mathcal{D}_n}} \right|}}{D}} \mathcal{L}\left( {w_n^*} \right)} \right]
\notag\\
&\le \mathbb{E}\left\{ {\mathop {\max }\limits_{\left( {{w_1},\cdots,{w_N}} \right)} \left[ {\sum\limits_{n = 1}^N {\frac{{\left| {{\mathcal{D}_n}} \right|}}{D}} \mathcal{L}\left( {{w_n}} \right) - \sum\limits_{n = 1}^N {\frac{{\left| {{\mathcal{D}_n}} \right|}}{D}} \mathcal{L}\left( {w_n^*} \right)} \right]} \right\} + \sqrt {\frac{{\rm C}^2D^2}{2}\log \frac{1}{\kappa }}.
\end{align}
We then calculate the upper bound of the first term in the right hand.
\begin{align}
    &\mathbb{E}\left\{ {\mathop {\max }\limits_{\left( {{w_1},\cdots,{w_N}} \right)} \left[ {\sum\limits_{n = 1}^N {\frac{{\left| {{\mathcal{D}_n}} \right|}}{D}} \mathcal{L}\left( {{w_n}} \right) - \sum\limits_{n = 1}^N {\frac{{\left| {{\mathcal{D}_n}} \right|}}{D}} \mathcal{L}\left( {w_n^*} \right)} \right]} \right\}
    \notag\\
     &\le \mathbb{E}\left\{ {\sum\limits_{n = 1}^N {\frac{{\left| {{\mathcal{D}_n}} \right|}}{D}\mathop {\max }\limits_{{w_n}} \left[ {\mathcal{L}\left( {{w_n}} \right) - \mathcal{L}\left( {w_n^*} \right)} \right]} } \right\}
     \notag\\
     &\le \sum\limits_{n = 1}^N {\frac{{\left| {{\mathcal{D}_n}} \right|}}{D}} \mathbb{E}\left\{ {\mathop {\sup }\limits_{f \in \mathcal{F} } \left[ {\mathcal{L}\left( {{w_n}} \right) - \mathcal{L}\left( {w_n^*} \right)} \right]} \right\}
     \notag\\
    &\mathop  = \limits^{(b)} \sum\limits_{n = 1}^N {\frac{{\left| {{\mathcal{D}_n}} \right|}}{D}} {\Re _n}(\mathcal{F} )
    \notag\\
     &\le \sum\limits_{n = 1}^N {\frac{{\left| {{\mathcal{D}_n}} \right|}}{D}} \sqrt {\frac{{2d}}{{\left| {{\mathcal{D}_n}} \right|}}\log \frac{{e\left| {{\mathcal{D}_n}} \right|}}{d}} 
     \notag\\
    &\mathop  \le \limits^{(c)} \sqrt {\frac{{2d}}{D}\log \frac{{eD}}{d}},
\end{align}
where $\mathcal{F}$ is the hypothesis set of the prediction model $f$ and $d$ is the VC-dimension of $\mathcal{F}$. (b) follows the definition of Rademacher complexity and Sauer's Lemma and (c) follows Jensen's inequality.

Therefore, 
\begin{align}
&\mathop {\max }\limits_{\left( {{w_1},\cdots,{w_N}} \right)} \left[ {\sum\limits_{n = 1}^N {\frac{{\left| {{\mathcal{D}_n}} \right|}}{D}} \mathcal{L}\left( {{w_n} } \right) - \sum\limits_{n = 1}^N {\frac{{\left| {{\mathcal{D}_n}} \right|}}{D}} \mathcal{L}\left( {w_n^*} \right)} \right]
\notag\\
&\le \sqrt {\frac{{2d}}{D}\log \frac{{eD}}{d}}  + \sqrt {\frac{{{{\rm C}^2}{D^2}}}{2}\log \frac{1}{\kappa }}.
\end{align}

\end{proof}

\subsection{Proof of Convergence Rate}
\label{proof2}
Let $w_{n,t}^i$ denote the local model parameters of client $n$ at the $i$-th iteration in the $t$-th FL round, where $0\le t\le T-1$ and $0\le i\le I-1$. $T$ and $I$ denote the total number of FL rounds and local iterations respectively.
$w_{n,t}^I$ represents the local model parameters of client $n$ after $I$ iterations in the $t$-th FL round. We denote $w_{n,t+1} = w_{n,t}^I$, which represents the local model of client $n$ in the $t+1$ round before the first iteration.
\begin{assumption}(\textit{Bounded Expectation of Gradients})
\label{G2}
The expectation of gradient of loss function $\mathcal{L}(\cdot)$ is uniformly bounded, i.e., ${\rm \mathbb{E}}(||\nabla {\mathcal{L}}( \cdot )|{|}) \le {G}.$
\end{assumption}
\begin{assumption}(\textit{Lipschitz Smooth})
\label{l-smooth2}
The loss function $\mathcal{L}$ is $L_1$-smooth, i.e., $\mathcal{L}(w) - \mathcal{L}(w') \le \left\langle {\nabla \mathcal{L}(w'), w-w'} \right\rangle  + L_1||w-w'|{|^2}, \forall w, w', \exists L_1 > 0.$
\end{assumption}
\begin{assumption}(\textit{Lipschitz Continuity})
Suppose $h:{\mathcal{A}_1} \times$
${\mathcal{A}_2} \times  \cdots  \to {\mathcal{A}_T}$ is Lipschitz Continuous in $\mathcal{A}_j$, i.e., $\exists L_h, \forall$
$a_j, \hat{a}_j \in \mathcal{A}_j,  \left\| {h({a_1}, \cdots ,{a_j}, \cdots ) - h({a_1}, \cdots ,{{\hat a}_j}, \cdots )} \right\| \le {L_h}\left\| {{a_j} - {{\hat a}_j}} \right\|$.
\label{continuous2}
\end{assumption}
\begin{theorem}(\textit{Convergence Rate})
    Let Assumption \ref{G2} to \ref{continuous2} hold. Let $\mathcal{L}\left( {w_{n,0}^0} \right) - \mathcal{L}\left( {w_n^*} \right) = \Lambda $.
    If clients adopt stochastic gradient descent method to optimize local prediction models with the learning rate equal to $\eta$, for any client, given $\xi >0$, after
    \begin{equation}
        T = \frac{\Lambda }{{\xi I(\eta  - {L_1}{\eta ^2}) - \rho \eta {L_h}N\alpha {I^2}G}}
        \notag
    \end{equation}
    FL rounds, we can obtain 
    \begin{equation}
        \frac{1}{{TI}}\sum\limits_{t = 0}^{T - 1} {\sum\limits_{i = 0}^{I - 1} {\mathbb{E}\left[ {{{\left\| {\nabla \mathcal{L}\left( {w_{n,t}^i} \right)} \right\|}^2}} \right]} }  < \xi 
        \notag
    \end{equation}
    with 
    \begin{equation}
        \eta < \frac{{\xi  - \rho {L_h}N\alpha IG}}{{{L_1}\xi }}.
        \notag
    \end{equation}
    
\end{theorem}

\begin{proof}
Let Assumption \ref{l-smooth2} hold, we have
\begin{align}
\mathcal{L}\left( {w_{n,t}^1} \right)\mathop  \le \mathcal{L}\left( {w_{n,t}^0} \right) + \left\langle {\nabla \mathcal{L}\left( {w_{n,t}^0} \right),\left( {w_{n,t}^1 - w_{n,t}^0} \right)} \right\rangle  + {L_1}{\left\| {w_{n,t}^1 - w_{n,t}^0} \right\|^2}.
\end{align}
With $w_{n,t}^1 = w_{n,t}^0 - \eta \nabla \mathcal{L}\left( {w_{n,t}^0} \right)$, it can be transformed into
\begin{align}
\mathcal{L}\left( {w_{n,t}^1} \right) &\le \mathcal{L}\left( {w_{n,t}^0} \right) - \eta \left\langle {\nabla \mathcal{L}\left( {w_{n,t}^0} \right),\nabla \mathcal{L}\left( {w_{n,t}^0} \right)} \right\rangle  + {L_1}{\eta ^2}{\left\| {\nabla \mathcal{L}\left( {w_{n,t}^0} \right)} \right\|^2}
\notag\\
&= \mathcal{L}\left( {w_{n,t}^0} \right) - \eta {\left\| {\nabla \mathcal{L}\left( {w_{n,t}^0} \right)} \right\|^2} + {L_1}{\eta ^2}{\left\| {\nabla \mathcal{L}\left( {w_{n,t}^0} \right)} \right\|^2}.
\end{align}
We can obtain 
\begin{align}
\mathcal{L}\left( {w_{n,t}^1} \right) &\le \mathcal{L}\left( {w_{n,t}^0} \right) - (\eta  - {L_1}{\eta ^2}){\left\| {\nabla \mathcal{L}\left( {w_{n,t}^0} \right)} \right\|^2}
\notag\\
\mathcal{L}\left( {w_{n,t}^2} \right) &\le \mathcal{L}\left( {w_{n,t}^1} \right) - (\eta  - {L_1}{\eta ^2}){\left\| {\nabla \mathcal{L}\left( {w_{n,t}^1} \right)} \right\|^2}
\notag\\
&\vdots 
\notag\\
\mathcal{L}\left( {{w_{n,t + 1}}} \right) = \mathcal{L}\left( {w_{n,t}^I} \right) &\le \mathcal{L}\left( {w_{n,t}^{I - 1}} \right) - (\eta  - {L_1}{\eta ^2}){\left\| {\nabla \mathcal{L}\left( {w_{n,t}^{I - 1}} \right)} \right\|^2}.
\notag
\end{align}
Therefore, 
\begin{equation}
\label{th1}
    \mathcal{L}\left( {{w_{n,t + 1}}} \right) \le \mathcal{L}\left( {w_{n,t}^0} \right) - (\eta  - {L_1}{\eta ^2})\sum\limits_{i = 1}^{I} {{{\left\| {\nabla \mathcal{L}\left( {w_{n,t}^{i - 1}} \right)} \right\|}^2}}.
\end{equation}
Furthermore, we have
\begin{align}
\label{tmp1}
\mathcal{L}\left( {w_{n,t + 1}^0} \right) = \mathcal{L}\left( {{w_{n,t + 1}}} \right) + \mathcal{L}\left( {w_{n,t + 1}^0} \right) - \mathcal{L}\left( {{w_{n,t + 1}}} \right).
\end{align}
With (16),
\begin{align}
&\mathcal{L}\left( {w_{n,t + 1}^0} \right) - \mathcal{L}\left( {{w_{n,t + 1}}} \right)
\notag\\
&= {\mathcal{L}_{pred}}\left( {{\delta _{n,t + 1}},{\phi _{n,t + 1}}} \right) + {\mathcal{L}_{ intra}}\left( {{\delta _{n,t + 1}}}, W_{n,t+1} \right) + \rho {\mathcal{L}_{ inter}}\left( {{\delta _{n,t + 1}};P{R_{n, t + 1}},N{R_{n, t + 1}}} \right) 
\notag\\
&- {\mathcal{L}_{pred}}\left( {{\delta _{n,t + 1}},{\phi _{n,t + 1}}} \right) - {\mathcal{L}_{ intra}}\left( {{\delta _{n,t + 1}}}, W_{n,t+1} \right) - \rho {\mathcal{L}_{inter}}\left( {{\delta _{n,t + 1}};P{R_{n,t}},N{R_{n,t}}} \right)
\notag\\
&= \rho \left\{ {{\mathcal{L}_{inter}}\left( {{\delta _{n,t + 1}};P{R_{n, t + 1}},N{R_{n, t + 1}}} \right) - {\mathcal{L}_{inter}}\left( {{\delta _{n,t + 1}};P{R_{n,t}},N{R_{n,t}}} \right)} \right\}.
\end{align}
Thereinto, $\delta_{n,t}$ denotes the encoder parameters of client $n$ at the $t$-th round. $PR_{n,t}$ and $NR_{n,t}$ represent the global positive and negative prototype for client $n$ respectively at the $t$-th round. We have
\begin{align}
{PR_{n,t}} &= \frac{1}{{\left| {{\mathcal{PC}_{n,t}}} \right|}}\sum\limits_{m \in {\mathcal{PC}_{n,t}}} {{R_{m,t}}} 
\notag\\
&= \frac{1}{{\left| {{\mathcal{PC}_{n,t}}} \right|}}\sum\limits_{m \in {\mathcal{PC}_{n,t}}} {\left[ {\frac{1}{I}\sum\limits_{i = 0}^{I - 1} {en\left( {\delta _{m,t - 1}^i} \right)} } \right]} 
\notag\\
&= \frac{1}{{I\left| {{\mathcal{PC}_{n,t}}} \right|}}\sum\limits_{m \in {\mathcal{PC}_{n,t}}} {\sum\limits_{i = 0}^{I - 1} {en\left( {\delta _{m,t - 1}^i} \right)} },
\end{align}
where $\mathcal{PC}_{n,t}$ denotes the positive set of client $n$ in the $t$-th round.
Similarly, 
\begin{align}
N{R_{n,t}} = \frac{1}{{I\left| {{\mathcal{NC}_{n,t}}} \right|}}\sum\limits_{m \in {\mathcal{NC}_{n,t}}} {\sum\limits_{i = 0}^{I - 1} {en\left( {\delta _{m,t - 1}^i} \right)} }.
\end{align}
Therefore, 
\begin{align}
&{\mathcal{L}_{ inter}}\left( {{\delta _{n,t + 1}};P{R_{n,t}},N{R_{n,t}}} \right)
\notag\\
&= {\mathcal{L}_{ inter}}\left( {{\delta _{n,t + 1}};\frac{1}{{I\left| {{\mathcal{PC}_{n,t}}} \right|}}\sum\limits_{m \in {\mathcal{PC}_{n,t}}} {\sum\limits_{i = 0}^{I - 1} {en\left( {\delta _{m,t - 1}^i} \right)} } ,\frac{1}{{I\left| {{\mathcal{NC}_{n,t}}} \right|}}\sum\limits_{m \in {\mathcal{NC}_{n,t}}} {\sum\limits_{i = 0}^{I - 1} {en\left( {\delta _{m,t - 1}^i} \right)} } } \right)
\notag\\
&{\buildrel \Delta \over =} h\left( {\delta _{1,t - 1}^0,\delta _{1,t - 1}^1, \cdots ,\delta _{1,t - 1}^{I - 1}, \cdots ,\delta _{N,t - 1}^0,\delta _{N,t - 1}^1, \cdots ,\delta _{N,t - 1}^{I - 1}} \right).
\end{align}
We can treat $\mathcal{L}_{inter}$ as a multivariate function $h$ with $NI$ variants, i.e., $\delta _{m,t - 1}^i, 1\le m\le N$, and $0\le i\le I-1$. 
We make the assumption that $h$ is Lipschitz continuous. Following Assumption \ref{continuous2}, we have $\exists {L_h} > 0$,
\begin{align}
&||h\left( {\delta _{1,t - 1}^0,\delta _{1,t - 1}^1, \cdots ,\delta _{1,t - 1}^{I - 1}, \cdots ,\delta _{N,t - 1}^0,\delta _{N,t - 1}^1, \cdots ,\delta _{N,t - 1}^{I - 1}} \right)
\notag\\
&- h\left( {\hat \delta _{1,t - 1}^0,\delta _{1,t - 1}^1, \cdots ,\delta _{1,t - 1}^{I - 1}, \cdots ,\delta _{N,t - 1}^0,\delta _{N,t - 1}^1, \cdots ,\delta _{N,t - 1}^{I - 1}} \right)||
\notag\\
&\le {L_h}\left\| {\delta _{1,t - 1}^0 - \hat \delta _{1,t - 1}^0} \right\|.
\end{align}
Therefore, 
\begin{align}
&{\mathcal{L}_{inter}}\left( {{\delta _{n,t + 1}};P{R_{t + 1}},N{R_{t + 1}}} \right) - {\mathcal{L}_{inter}}\left( {{\delta _{n,t + 1}};P{R_{n,t}},N{R_{n,t}}} \right)
\notag\\
&= h\left( {\delta _{1,t}^0,\delta _{1,t}^1, \cdots ,\delta _{1,t}^{I - 1}, \cdots ,\delta _{N,t}^0,\delta _{N,t}^1, \cdots ,\delta _{N,t}^{I - 1}} \right) - h\left( {\delta _{1,t - 1}^0,\delta _{1,t - 1}^1, \cdots ,\delta _{1,t - 1}^{I - 1}, \cdots ,\delta _{N,t - 1}^0,\delta _{N,t - 1}^1, \cdots ,\delta _{N,t - 1}^{I - 1}} \right)
\notag\\
&\le \left\| {h\left( {\delta _{1,t}^0,\delta _{1,t}^1, \cdots ,\delta _{1,t}^{I - 1}, \cdots ,\delta _{N,t}^0,\delta _{N,t}^1, \cdots ,\delta _{N,t}^{I - 1}} \right) - h\left( {\delta _{1,t - 1}^0,\delta _{1,t - 1}^1, \cdots ,\delta _{1,t - 1}^{I - 1}, \cdots ,\delta _{N,t - 1}^0,\delta _{N,t - 1}^1, \cdots ,\delta _{N,t - 1}^{I - 1}} \right)} \right\|
\notag\\
&\le {L_h}\left( {\sum\limits_{m \in \mathcal{SC}_t} {\sum\limits_{i = 0}^{I - 1} {\left\| {\delta _{m,t}^i - \delta _{m,t - 1}^i} \right\|} } } \right),
\end{align}
where $\mathcal{SC}_t$ represents the set of selected $N\alpha$ clients in the $t$-th FL round.

Therefore, (\ref{tmp1}) can be transformed into 
\begin{align}
\mathcal{L}\left( {w_{n,t + 1}^0} \right) &= \mathcal{L}\left( {{w_{n,t + 1}}} \right) + \mathcal{L}\left( {w_{n,t + 1}^0} \right) - \mathcal{L}\left( {{w_{n,t + 1}}} \right)
\notag\\
&\le \mathcal{L}\left( {{w_{n,t + 1}}} \right) + \rho {L_h}\left( {\sum\limits_{m \in {\mathcal{SC}_t}} {\sum\limits_{i = 0}^{I - 1} {\left\| {\delta _{m,t}^i - \delta _{m,t - 1}^i} \right\|} } } \right)
\notag\\
&= \mathcal{L}\left( {{w_{n,t + 1}}} \right) + \rho {L_h}\sum\limits_{m \in {\mathcal{SC}_t}} {\sum\limits_{i = 0}^{I - 1} {\left\| {\sum\limits_{j = i}^{I - 1} {\eta \Delta w_{m,t - 1}^j}  + \sum\limits_{j = 0}^{i - 1} {\eta \Delta w_{m,t}^j} } \right\|} } .
\end{align}
Take expectations on both sides, we have
\begin{align}
\label{th2}
\mathbb{E}\left[ {\mathcal{L}\left( {w_{n,t + 1}^0} \right)} \right] &\le \mathcal{L}\left( {{w_{n,t + 1}}} \right) + \rho {L_h}\sum\limits_{m \in {\mathcal{SC}_t}} {\sum\limits_{i = 0}^{I - 1} {\left\| {\sum\limits_{j = i}^{I - 1} {\eta \mathbb{E}\left[ {\Delta w_{m,t - 1}^i} \right]}  + \sum\limits_{j = 0}^{i - 1} {\eta \mathbb{E}\left[ {\Delta w_{m,t}^i} \right]} } \right\|} } 
\notag\\
&\le \mathcal{L}\left( {{w_{n,t + 1}}} \right) + \rho {L_h}\sum\limits_{m \in {\mathcal{SC}_t}} {\sum\limits_{i = 0}^{I - 1} {\eta IG} } 
\notag\\
&\le \mathcal{L}\left( {{w_{n,t + 1}}} \right) + \rho \eta {L_h}N\alpha {I^2}G
\end{align}
By incorporating (\ref{th1}) with (\ref{th2}), we can obtain
\begin{equation}
\label{th3}
\mathbb{E}\left[ {\mathcal{L}\left( {w_{n,t + 1}^0} \right)} \right] \le \mathcal{L}\left( {w_{n,t}^0} \right) + \rho \eta {L_h}N\alpha {I^2}G - (\eta  - {L_1}{\eta ^2})\sum\limits_{i = 0}^{I - 1} {{{\left\| {\nabla \mathcal{L}\left( {w_{n,t}^i} \right)} \right\|}^2}} .
\end{equation}

To ensure 
\begin{equation}
\rho \eta {L_h}N\alpha {I^2}G - (\eta  - {L_1}{\eta ^2})\sum\limits_{i = 0}^{I - 1} {{{\left\| {\nabla L\left( {w_{n,t}^i} \right)} \right\|}^2}}  \le 0,
\end{equation}
we can obtain
\begin{equation}
\eta  \le \frac{{\sum\limits_{i = 0}^{I - 1} {{{\left\| {\nabla \mathcal{L}\left( {w_{n,t}^i} \right)} \right\|}^2} - \rho {L_h}N\alpha {I^2}G} }}{{{L_1}\sum\limits_{i = 0}^{I - 1} {{{\left\| {\nabla \mathcal{L}\left( {w_{n,t}^i} \right)} \right\|}^2}} }}.
\end{equation}

Considering all the FL rounds, i.e., $0\le t\le T-1$, with (\ref{th3}), we have
\begin{equation}
\sum\limits_{t = 0}^{T - 1} {\mathbb{E}\left[ {\mathcal{L}\left( {w_{n,t + 1}^0} \right)} \right]}  \le \sum\limits_{t = 0}^{T - 1} {\mathcal{L}\left( {w_{n,t}^0} \right)}  + T\rho \eta {L_h}N\alpha {I^2}G - (\eta  - {L_1}{\eta ^2})\sum\limits_{t = 0}^{T - 1} {\sum\limits_{i = 0}^{I - 1} {\mathbb{E}\left[ {{{\left\| {\nabla \mathcal{L}\left( {w_{n,t}^i} \right)} \right\|}^2}} \right]} } ,
\end{equation}
which can be transformed into 
\begin{equation}
\frac{1}{{TI}}\sum\limits_{t = 0}^{T - 1} {\sum\limits_{i = 0}^{I - 1} {\mathbb{E}\left[ {{{\left\| {\nabla \mathcal{L}\left( {w_{n,t}^i} \right)} \right\|}^2}} \right]} }  \le \frac{{\frac{1}{{TI}}\sum\limits_{t = 0}^{T - 1} {\left\{ {\mathcal{L}\left( {w_{n,t}^0} \right) - E\left[ {\mathcal{L}\left( {w_{n,t + 1}^0} \right)} \right]} \right\}}  + \rho \eta {L_h}N\alpha IG}}{{\eta  - {L_1}{\eta ^2}}}.
\end{equation}
Suppose $\mathcal{L}\left( {w_{n,0}^0} \right) - \mathcal{L}\left( {w_n^*} \right) = \Lambda $, 
\begin{equation}
\frac{{\frac{1}{{TI}}\sum\limits_{t = 0}^{T - 1} {\left\{ {\mathcal{L}\left( {w_{n,t}^0} \right) - E\left[ {\mathcal{L}\left( {w_{n,t + 1}^0} \right)} \right]} \right\}}  + \rho \eta {L_h}N\alpha IG}}{{\eta  - {L_1}{\eta ^2}}} < \frac{{\frac{\Lambda }{{TI}} + \rho \eta {L_h}N\alpha IG}}{{\eta  - {L_1}{\eta ^2}}} < \xi .
\end{equation}
Therefore, we can obtain that, with
\begin{equation}
T > \frac{\Lambda }{{\xi I(\eta  - {L_1}{\eta ^2}) - \rho \eta {L_h}N\alpha {I^2}G}},
\end{equation}
\begin{equation}
\frac{1}{{TI}}\sum\limits_{t = 0}^{T - 1} {\sum\limits_{i = 0}^{I - 1} {\mathbb{E}\left[ {{{\left\| {\nabla \mathcal{L}\left( {w_{n,t}^i} \right)} \right\|}^2}} \right]} }  < \xi,
\end{equation}
if 
\begin{equation}
\eta  \le \frac{{\sum\limits_{i = 0}^{I - 1} {{{\left\| {\nabla \mathcal{L}\left( {w_{n,t}^i} \right)} \right\|}^2} - \rho {L_h}N\alpha {I^2}G} }}{{{L_1}\sum\limits_{i = 0}^{I - 1} {{{\left\| {\nabla \mathcal{L}\left( {w_{n,t}^i} \right)} \right\|}^2}} }} < \frac{{I\xi  - \rho {L_h}N\alpha {I^2}G}}{{{L_1}I\xi }} = \frac{{\xi  - \rho {L_h}N\alpha IG}}{{{L_1}\xi }}.
\end{equation}

\end{proof}

\begin{table}[htbp]
    \small
  \centering
  \caption{Comparison of FUELS and three baselines in terms of communication parameters (\textit{Comm}) and computation cost (\textit{Comp}). $|w|$ and $|\delta|$ denote the number of parameters in prediction model and in the encoder respectively. $I'$ denotes the number of batches of the test dataset. $\mathcal{T}$ in PerFedAvg denotes the number of local gradient descent steps.}
  \vskip 0.1in
    \begin{tabular}{cccccc}
    \toprule
          & \multicolumn{2}{c}{\textit{Comm}} & \multicolumn{3}{c}{\textit{Comp}} \\
          & Client  & Server & Train (Client) & Infer (Client) & Train (Server) \\
    \midrule
    FedAvg & $\mathcal{O}(\left| w \right|)$     & $\mathcal{O}(N\alpha \left| w \right|)$     & $\mathcal{O}(2I(FE + FD))$     & $\mathcal{O}(I'(FE + FD))$     & $\mathcal{O}(1)$ \\
    FedRep & $\mathcal{O}(\left| \delta  \right|)$     & $\mathcal{O}(N\alpha \left| \delta  \right|)$     & $\mathcal{O}(2I(FE + FD))$     & $\mathcal{O}(I'(FE + FD))$     & $\mathcal{O}(1)$ \\
    PerFedAvg & $\mathcal{O}(\left| w \right|)$    & $\mathcal{O}(N\alpha \left| w \right|)$     & $\mathcal{O}(6\mathcal{T} (FE + FD))$     & $\mathcal{O}((I' + 2)(FE + FD))$     & $\mathcal{O}(1)$ \\
    \rowcolor{blue!40} FUELS & $\mathcal{O}(B{d_r})$     & $\mathcal{O}(2N\alpha B{d_r})$     & $\mathcal{O}(I(3FE + 2FF + 2FD))$     & $\mathcal{O}(I'(FE + FD))$     & $\mathcal{O}\left( {\left( {\alpha  - \frac{{{\alpha ^2}}}{2}} \right){N^2} - \frac{\alpha }{2}N} \right)$ \\
    \bottomrule
    \end{tabular}%
  \label{complexity}%
  \vskip -0.1in
\end{table}%
\subsection{Complexity Analysis}
\label{analysis}
We compare the computation cost and the number of communication parameters in FUELS with three personalized FL methods, which is presented in Table \ref{complexity}.
\textit{Comm} represents the number of parameters each client (the server) transmitted to the server (all the selected clients) each round.
The 4-th and 6-th rows in Table \ref{complexity} present the computation cost of a client and the server in each FL training round respectively. \textit{Comp} in the inference stage represents the computation cost for inferring the total test dataset.

Compared with the other methods based on model transmission, FUELS is highly communication-efficient, i.e., $B{d_r} \ll \left| \delta  \right| < \left| w \right|$.
In the forward propagation, FUELS should encode the augmented data for the subsequent intra-client contrastive task, and hence resulting in more computation cost in the training stage.
The server should update the JSD matrix according to the fresh local prototypes to generate the customized global prototypes for the next round. We deem that the introduced computation cost at the server can be overlooked, given the abundant resources of the server.

\section{Experimental Details and Additional Results}

\subsection{Experimental Details}
\label{exprimnt}
\textbf{Datasets.} The adopted dataset is from the Big Data Challenge programmed by the Telecom Italia \cite{barlacchi2015multi}. The area of Trentino is divided into 6575 cells with size of $235\times235$ meters, which is shown in Figure \ref{map}.
The dataset includes three types of Call Detail Records (CDRs), i.e., Short Message Services (SMS), Voice Calls (Call), and Internet Services (Net), detected from 2013/11/01 to 2014/01/01 every 10 minutes. 
The diversity of traffic data from different cells exhibits the spatial heterogeneity. 
For example, the traffic data of the cells in different function zones may have diverse temporal patterns, i.e., traffic distribution in downtown area is quite different from that in suburb area.

\begin{figure}[ht]
\vskip 0.1in
    \centering
    \includegraphics[width=0.7\linewidth]{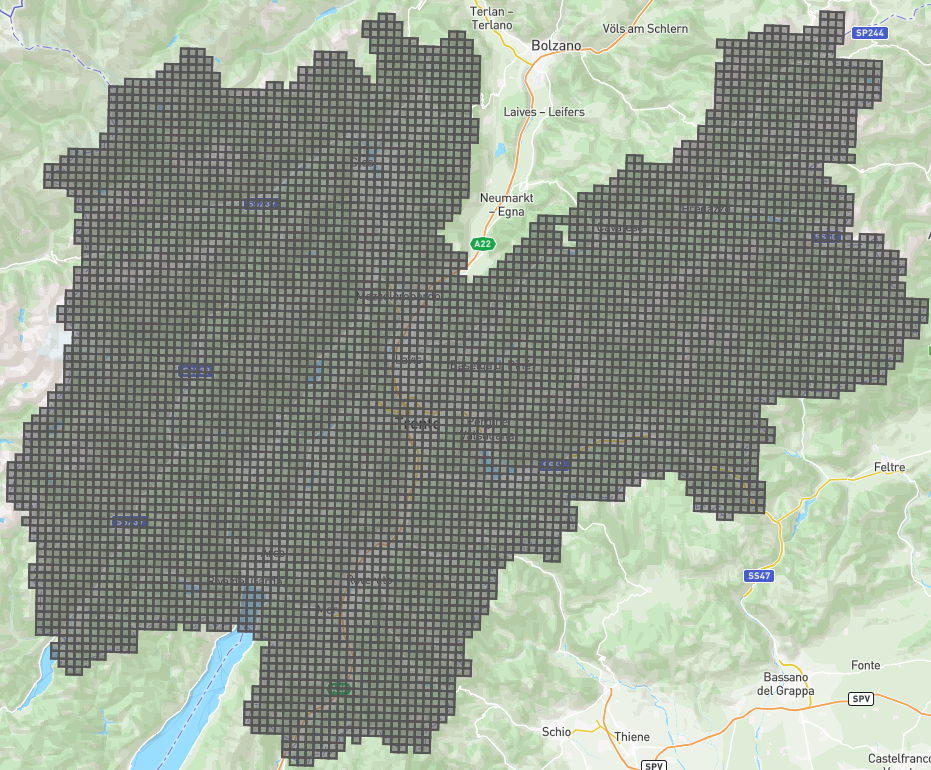}
    \caption{The division of Trentino.}
    \label{map}
\vskip -0.1in
\end{figure}
\textbf{Baselines.} We compare the performance of FUELS with the following baselines.
\begin{itemize}[leftmargin=*]
    \item \textbf{Solo}: All clients train their own local prediction models using local data and do not communicate with each other or with the server.
    \item \textbf{FedAvg} \cite{mcmahan2017communication}: It is a classical FL method, where the server aggregates the local model updates via averaging strategy.
    \item \textbf{FedProx} \cite{li2020federated}: Compared with FedAvg, a proximal term is added to the local loss function for tackling with the statistical heterogeneity.
    \item \textbf{FedRep} \cite{collins2021exploiting}: In this method, only parameters of low layers (encoder in our paper) are communicated between clients and the server for learning shared representations. Each client maintains its personalized head (decoder in our paper) for the final prediction.
    \item \textbf{PerFedAvg} \cite{fallah2020personalized}: This method aims to train an initial model which can quickly adapt to clients' local data via several steps of gradient descent. 
    \item \textbf{pFedMe} \cite{DBLP:conf/nips/DinhTN20}: In this method, by the designed regularized loss function, each client can train the global model and simultaneously optimize the personalized model.
    \item \textbf{FedDA} \cite{zhang2021dual}: In this method, the server firstly train a initial global model via the uploaded augmented data from clients. In the FL training stage, the server aggregates local models by the designed dual attention mechanism.
\end{itemize}

\textbf{Experimental Setup.}
To tackle with the problem of data sparsity, traffic data are resampled every 1 hour as per \cite{zhang2021dual}. Without loss of generality, we randomly select 100 cells as the clients, i.e., $N=100$.
The experimental setup of FedDA complies with the corresponding literature, unless stated otherwise.
The experimental settings of the other baselines are the same as FUELS for the sake of fairness.
In our experiments, the $\rm GRU_c$ and $\rm GRU_p$ all have 1 GRU layer with 128 cells.
The decoder is a fully-connected layer.
The batch size $B$ is set to 24. The closeness window $c$ and periodic window $q$ are both set to 3. The temperature rate $\tau$ is set to $0.02$.
For each dataset, we set $\beta$ as the 50-th percentile of all JSD values.
The training round of FUELS $T$ is set to 200, and each client transmits periodicity-aware prototypes to the server after conducting local training for 1 epoch at each round.
We use an Adam optimizer with the learning rate 0.001 \cite{kingma2014adam} .

\subsection{Privacy Protection}
\label{sec_privacy}
The widely adopted differential privacy in FL framework for preserving privacy is based on adding random noise with $Laplace$, $Guassian$, or $exponential$ distribution to the uploaded model parameters.
However, this approach may perturb the model parameters, thus worsening the final prediction performance.
Similarly, we add various random noise to the communicated local prototypes. The performance comparison is presented in Table \ref{privacy}. We observe that FUELS with different types of random noise can always achieve the superior performance, compared with FedAvg. Moreover, FUELS promises the combination of privacy-preserving techniques without significant performance decrease.
\begin{table}[htbp]
  \centering
  \caption{The performance of FUELS with different types of privacy-preserving mechanisms. $\mu$ and $\lambda$ denote the mean value and the standard deviation of the corresponding noise distribution.}
  \vskip 0.1in
    \begin{tabular}{cccccccccc}
    \toprule
    Dataset & \multicolumn{2}{c}{SMS} & \multicolumn{2}{c}{Call} & \multicolumn{2}{c}{Net} \\
    Metric & MSE     & MAE   & MSE     & MAE   & MSE    & MAE \\
    \midrule
    \rowcolor{yellow!40} FedAvg & 1.452   & 0.533  & 0.393    & 0.300  & 2.649    & 0.638  \\
    \rowcolor{blue!40}FUELS & 1.249    & 0.541  & 0.353    & 0.311  & 0.880    & 0.488  \\
    \rowcolor{pink!40} FUELS + $Laplace(\mu=0,\lambda=1)$ & 1.397    & 0.563  & 0.443    & 0.335  & 1.086    & 0.531  \\
    \rowcolor{pink!40} FUELS + $Guassian(\mu=0,\lambda=1)$ & 1.395    & 0.564  & 0.419    & 0.328  & 1.117  &  0.531  \\
    \rowcolor{pink!40} FUELS + $exponential(\lambda=1)$   & 1.365    & 0.554  & 0.406    & 0.320  & 1.047    & 0.514  \\
    \bottomrule
    \end{tabular}%
  \label{privacy}%
  \vskip -0.1in
\end{table}%

\subsection{Effect of Prototype Size}
\label{append-proto}
The performance of FUELS with different setting of $d_r$ is presented in Table \ref{pro_size}. 
Prototype size can be adjusted for different prediction tasks. In general, better performance can be achieved with $d_r$ increasing, which on the other hand results in more computing and communication overhead.

\begin{table}[htbp]
  \centering
  \caption{The performance of FUELS with different settings of prototype size.}
  \vskip 0.1in
    \begin{tabular}{cccccccc}
    \toprule
    \multirow{2}[2]{*}{$d_r$} & \multicolumn{2}{c}{SMS} & \multicolumn{2}{c}{Call} & \multicolumn{2}{c}{Net} & \multicolumn{1}{c}{\multirow{2}[4]{*}[+1.5ex]{\makecell[c]{\# of Comm \\ Params} }} \\
          & MSE   & MAE   & MSE   & MAE   & MSE   & MAE   &  \\
    \midrule
    2     & 3.588  & 1.125  & 1.692  & 0.848  & 4.500  & 1.213  & 96 \\
    4     & 2.910  & 0.899  & 1.265  & 0.634  & 3.899  & 1.041  & 192 \\
    8     & 2.060  & 0.683  & 0.775  & 0.429  & 2.559  & 0.791  & 384 \\
    16    & 1.709  & 0.611  & 0.568  & 0.364  & 2.039  & 0.672  & 768 \\
    32    & 1.445  & 0.567  & 0.471  & 0.340  & 1.289  & 0.562  & 1536 \\
    64    & 1.348  & 0.551  & 0.412  & 0.321  & 1.051  & 0.517  & 3072 \\
    \rowcolor{blue!40} 128   & 1.249  & 0.541  & 0.353  & 0.311  & 0.880  & 0.488  & 6144 \\
    256   & 1.272  & 0.531  & 0.318  & 0.298  & 0.997  & 0.487  & 12288 \\
    512   & 1.248  & 0.538  & 0.326  & 0.300  & 0.990  & 0.486  & 24576 \\
    \bottomrule
    \end{tabular}%
  \label{pro_size}%
  \vskip -0.1in
\end{table}%


\subsection{Results on Other Dataset}
\label{other}
We conduct further experiments on the METR-LA, a widely-used dataset in the task of traffic flow forecasting. It includes traffic speed records of 207 loop detectors deployed in Los Angeles from 2012/01/01 to 2012/06/30 every 5 minutes. The experimental settings on METR-LA are the same as those on the other three datasets in Section \ref{perfom-set}. 
The experimental results of FUELS and baselines are presented in Table \ref{metr}. It is obvious that FUELS can also achieve the dominate performance compared with the baselines.

\begin{table}[htbp]
  \centering
  \caption{Numerical results on METR-LA dataset.}
  \vskip 0.1in
    \begin{tabular}{ccccccc}
    \toprule
    Metric & FedAvg & FedProx & FedRep & PerFedAvg & pFedMe & FUELS \\
    \midrule
    MSE   & 0.126  & 0.127  & 0.127  & 0.143  & 0.137  & 0.125  \\
    MAE   & 0.167  & 0.166  & 0.164  & 0.193  & 0.173  & 0.166  \\
    \bottomrule
    \end{tabular}%
  \label{metr}%
  \vskip -0.1in
\end{table}%

\begin{table}[htbp]
  \centering
  \caption{Numerical results of FUELS and four variants on SMS, Call, and Net datasets.}
  \vskip 0.1in
    \begin{tabular}{ccccccc}
    \toprule
    Dataset & \multicolumn{2}{c}{SMS} & \multicolumn{2}{c}{Call} & \multicolumn{2}{c}{Net} \\
    Metric & {MSE} & {MAE} & {MSE} & {MAE} & {MSE} &  {MAE} \\
    \midrule
    \rowcolor{yellow!40}w/o inter & 1.340   & 0.560  & 0.392  & 0.323  & 0.983   & 0.494  \\
    \rowcolor{yellow!40}w/o intra & 1.336   & 0.566  & 0.455  & 0.334  & 1.000  & 0.488  \\
    \rowcolor{red!40}w/o p-aware & 1.324  & 0.575  & 0.400   & 0.342  & 1.252   & 0.570  \\
    \rowcolor{pink!40}w/o \textit{W} & 1.480   & 0.582  & 0.445   & 0.337  & 1.627  & 0.633  \\
    \rowcolor{blue!40}FUELS & 1.249  & 0.541  & 0.353   & 0.311  & 0.880  & 0.488  \\
    \bottomrule
    \end{tabular}%
  \label{ab_full}%
  \vskip -0.1in
\end{table}%

\begin{figure}[ht]
\vskip 0.1in
    \centering
    \includegraphics[width=0.8\linewidth]{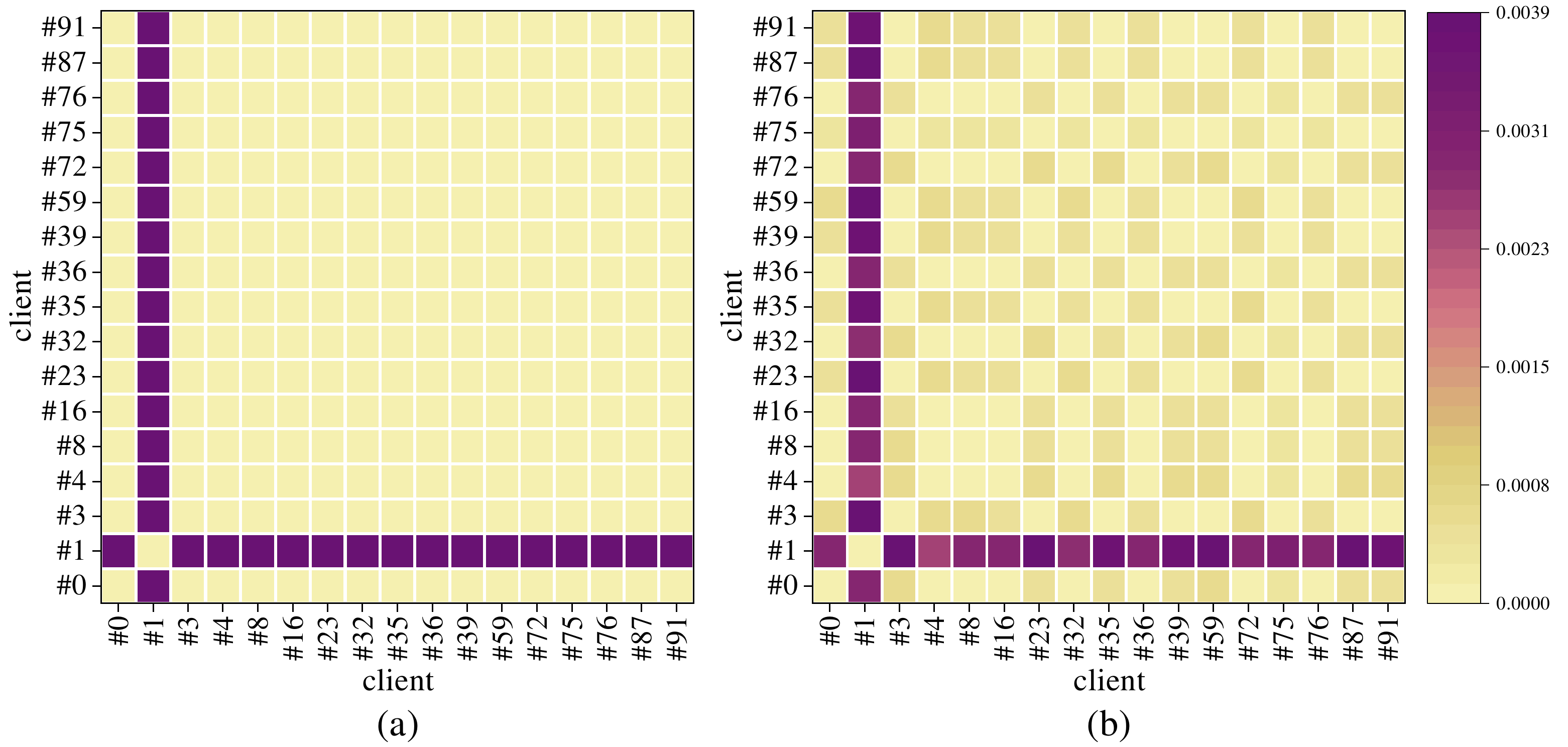}
    \caption{Results of JSD results for (a) local prototypes and (b) raw traffic data. In (a), if the JSD value of two local prototypes is less than $\beta$, the corresponding space is marked in yellow, otherwise in purple.}
    \label{js}
    \vskip -0.1in
\end{figure}

\begin{figure}[ht]
\vskip 0.1in
    \centering
    \includegraphics[width=0.75\linewidth]{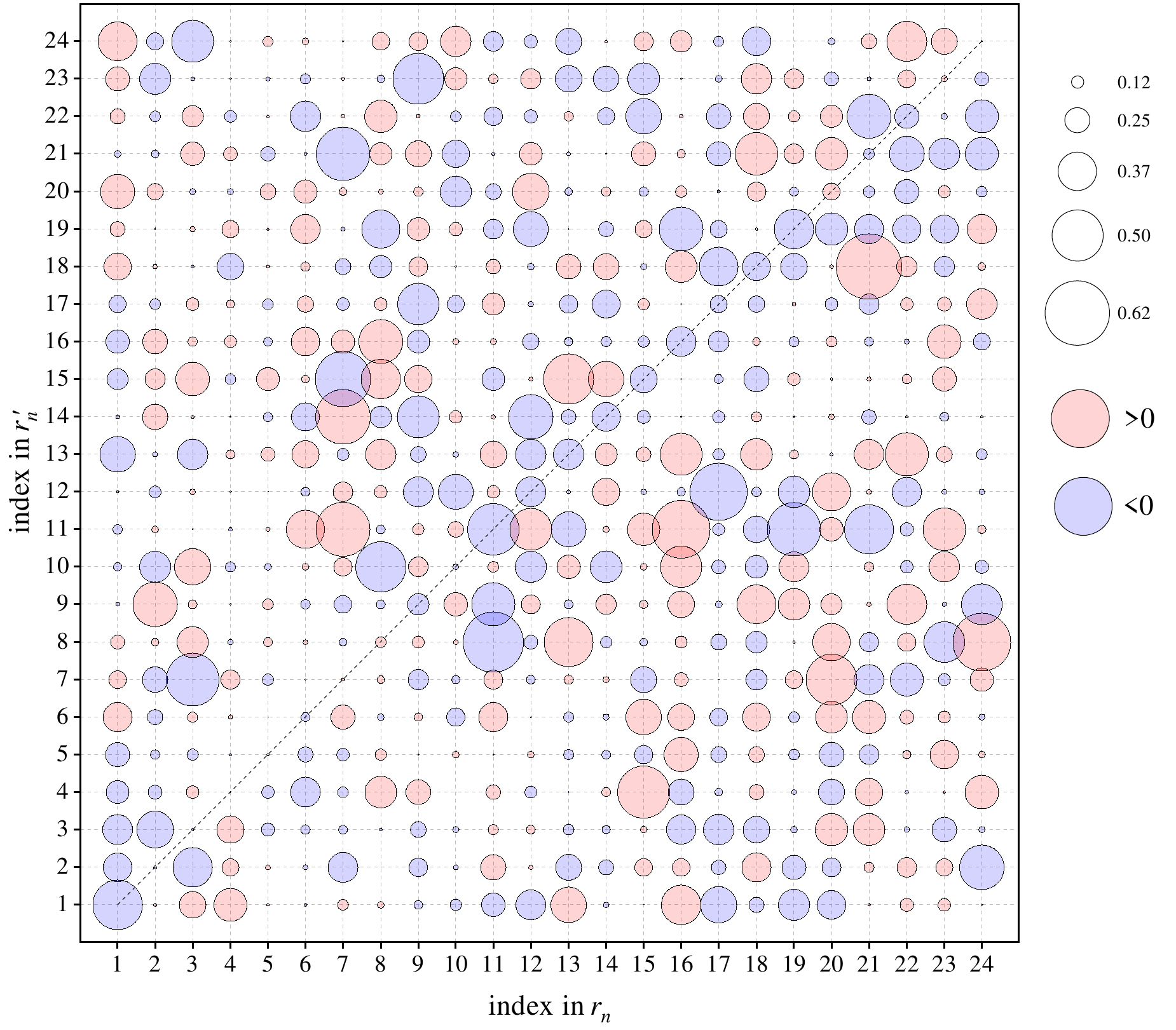}
    \caption{Visualization of $W_n$. The circle size represents the absolute value of a parameter. If a parameter is over 0, it is marked in pink, otherwise in purple.}
    \label{filter}
\vskip -0.1in
\end{figure}

\end{document}